\documentclass[letterpaper,11pt]{article}

\usepackage[abbrvbib,usehyper]{jmlr_adapted}
\usepackage{natbib}
\usepackage{amsmath}
\usepackage{graphicx}
\usepackage{booktabs} 
\usepackage{nicefrac}
\usepackage[ruled]{algorithm2e}
\usepackage{multirow}
\usepackage{eucal}
\usepackage{mpemath}
\usepackage{url}
\usepackage{tabularx}
\usepackage[dvipsnames]{xcolor}
\usepackage{enumitem}
\usepackage{mathabx}
\usepackage{tikz}

\hypersetup{colorlinks=true,citecolor=Gray,linkcolor=Black,urlcolor=Black}
 
\newcommand{\states}{\mathcal{S}}
\newcommand{\actions}{\mathcal{A}}
\newcommand{\probs}{\mathcal{P}}  

\newcommand{\sU}{\mathcal{U}}
\newcommand{\sL}{\mathcal{L}}
\newcommand{\sZ}{ {\widebar{\mathcal{N}}}}
\newcommand{\sUL}{\mathcal{E}}
\newcommand{\sUZ}{{\widebar{\mathcal{U}}}}
\newcommand{\sLZ}{{\widebar{\mathcal{L}}}}

\newcommand{\sQ}{\mathcal{Q}}
\newcommand{\sN}{\mathcal{N}} 
\newcommand{\sE}{\sUL}

\renewcommand{\b}[1]{\bm{#1}}
\renewcommand{\*}[1]{\bm{#1}}

\newcommand{\cm}{$\checkmark$}
\newcommand{\nm}{$\cdot$}

\newcommand{\Bell}{\mathfrak{L}}
\newcommand{\statecount}{S}
\newcommand{\actioncount}{A}

\DeclareRobustCommand{\bigO}{\text{\usefont{OMS}{cmsy}{m}{n}O}}

\renewcommand{\cite}[1]{\citep{#1}}

\title{Partial Policy Iteration \\for $L_1$-Robust Markov Decision Processes}

\author{%
	\name Chin Pang Ho \email clint.ho@cityu.edu.hk \\
	City University of Hong Kong \AND
	\name Marek Petrik \email mpetrik@cs.unh.edu \\
	University of New Hampshire \AND
	\name Wolfram Wiesemann \email ww@imperial.ac.uk \\
	Imperial College London
}

\begin{document}

\maketitle

\begin{abstract}%
    Robust Markov decision processes (MDPs) allow to compute reliable solutions for dynamic decision problems whose evolution is modeled by rewards and partially-known transition probabilities. Unfortunately, accounting for uncertainty in the transition probabilities significantly increases the computational complexity of solving robust MDPs, which severely limits their scalability. This paper describes new efficient algorithms for solving the common class of robust MDPs with s- and sa-rectangular ambiguity sets defined by weighted $L_1$ norms. We propose partial policy iteration, a new, efficient, flexible, and general policy iteration scheme for robust MDPs. We also propose fast methods for computing the robust Bellman operator in quasi-linear time, nearly matching the linear complexity the non-robust Bellman operator. Our experimental results indicate that the proposed methods are many orders of magnitude faster than the state-of-the-art approach which uses linear programming solvers combined with a robust value iteration.
\end{abstract}

\section{Introduction}

Markov decision processes (MDPs) provide a versatile methodology for modeling and solving dynamic decision problems under uncertainty~\cite{Puterman2005}. Unfortunately, however, MDP solutions can be very sensitive to estimation errors in the transition probabilities and rewards. This is of particular worry in \emph{reinforcement learning} applications, where the model is fit to data and therefore inherently uncertain. Robust MDPs (RMDPs) do not assume that the transition probabilities are known precisely but instead allow them to take on any value from a given \emph{ambiguity set} or uncertainty set~\cite{Xu2006,Mannor2012,Hanasusanto2013,Tamar2014a,Delgado2016}. With appropriately chosen ambiguity sets, RMDP solutions are often much less sensitive to model errors~\cite{Xu2009,Petrik2012,Petrik2016a}. 


Most of the RMDP literature assumes \emph{rectangular} ambiguity sets that constrain the errors in the transition probabilities independently for each state~\cite{Iyengar2005,Nilim2005,LeTallec2007,Kaufman2013,Wiesemann2013}. This assumption is crucial to retain many of the desired structural features of MDPs. In particular, the robust return of an RMDP with a rectangular ambiguity set is maximized by a \emph{stationary policy}, and the optimal value function satisfies a robust variant of the Bellman optimality equation. Rectangularity also ensures that an optimal policy can be computed in \emph{polynomial time} by robust versions of the classical value or policy iteration~\cite{Iyengar2005,Hansen2013}.

A particularly popular class of rectangular ambiguity sets is defined by bounding the $L_1$-distance of any plausible transition probabilities from a \emph{nominal} distribution~\cite{Iyengar2005,Strehl2009,Auer2010,Petrik2014,Taleghan2015,Petrik2016a}. Such ambiguity sets can be readily constructed from samples~\cite{Weissman2003xx,Behzadian2019}, and their polyhedral structure implies that the worst transition probabilities can be computed by the solution of linear programs (LPs). Unfortunately, even for the specific class of $L_1$-ambiguity sets, an LP has to be solved for each state and each step of the value or policy iteration. Generic LP algorithms have a worst-case complexity that is approximately quartic in the number of states~\cite{Vanderbei2001}, and they thus become prohibitively expensive for RMDPs with many states.

In this paper, we propose a new framework for solving RMDPs. Our framework applies to both sa-rectangular ambiguity sets, where adversarial nature observes the agent's actions before choosing the worst plausible transition probabilities~\cite{Iyengar2005,Nilim2005}, and s-rectangular ambiguity sets, where nature must commit to a realization of the transition probabilities before observing the agent's actions~\cite{LeTallec2007,Wiesemann2013}. We achieve a significant theoretical and practical acceleration over the robust value and policy iteration by reducing the number of iterations needed to compute an optimal policy and by reducing the computational complexity of each iteration. The overall speedup of our framework allows us to solve RMDPs with $L_1$-ambiguity sets in a time complexity that is similar to that of classical MDPs. Our framework comprises of three components, each of which represents a novel contribution.

Our first contribution is \emph{partial policy iteration} (PPI), which generalizes the classical modified policy iteration to RMDPs. PPI resembles the robust modified policy iteration~\cite{Kaufman2013}, which has been proposed for sa-rectangular ambiguity sets. In contrast to the robust modified policy iteration, however, PPI applies to both sa-rectangular and s-rectangular ambiguity sets, and it is guaranteed to converge at the same linear rate as robust value and robust policy iteration. In our experimental results, PPI outperforms robust value iteration by several orders of magnitude.

Our second contribution is a fast algorithm for computing the robust Bellman operator for sa-rectangular weighted $L_1$-ambiguity sets.  Our algorithm employs the \emph{homotopy continuation} strategy~\cite{Vanderbei2001}: it starts with a singleton ambiguity set for which the worst transition probabilities can be trivially identified, and it subsequently traces the most adverse transition probabilities as the size of the ambiguity set increases. The time complexity of our homotopy method is quasi-linear in the number of states and actions, which is significantly faster than the quartic worst-case complexity of generic LP solvers.

Our third contribution is a fast algorithm for computing the robust Bellman operator for s-rectangular weighted $L_1$-ambiguity sets. While often less conservative and hence more appropriate in practice, s-rectangular ambiguity sets are computationally challenging since the agent's optimal policy can be randomized~\cite{Wiesemann2013}. We propose a \emph{bisection} approach to decompose the s-rectangular Bellman computation into a series of sa-rectangular Bellman computations. When our bisection method is combined with our homotopy method, its time complexity is quasi-linear in the number of states and actions, compared again to the quartic complexity of generic LP solvers.

Put together, our contributions comprise a complete framework that can be used to solve RMDPs efficiently. Besides being faster than solving LPs directly, our framework does not require an expensive black-box commercial optimization package such as CPLEX, Gurobi, or Mosek. A well-tested and documented implementation of the methods described in this paper is available at \url{https://github.com/marekpetrik/craam2}.

Compared to an earlier conference version of this work~\cite{Ho2018}, the present paper introduces PPI, it improves the bisection method to work with PPI, it provides extensive and simpler proofs, and it reports more complete and thorough experimental results. 

The remainder of the paper is organized as follows. We summarize relevant prior work in \cref{sec:prior_work} and subsequently review  basic properties of RMDPs in \cref{sec:robust_bellman_updates}. \Cref{sec:ppi} describes our partial policy iteration~(PPI), \cref{sec:fast_solution} develops the homotopy method for sa-rectangular ambiguity sets, and \cref{sec:decomposition} is devoted to the bisection method for s-rectangular ambiguity sets. \cref{sec:numerical} compares our algorithms with the solution of RMDPs via Gurobi, a leading commercial LP solver, and we offer concluding remarks in \cref{sec:conclusion}.

\textbf{Notation.}
Regular lowercase letters (such as $p$) denote scalars, boldface lowercase letters (such as $\b{p}$) denote vectors, and boldface uppercase letters (such as $\b X$) denote matrices. Indexed values are printed in bold if they are vectors and in regular font if they are scalars. That is,  $p_i$ refers to the $i$-th element of a vector $\*p$, whereas $\*p_i$ is the $i$-th vector of a sequence of vectors. An expression in parentheses indexed by a set of natural numbers, such as $(p_i)_{i\in\mathcal{Z}}$ for $\mathcal{Z} = \{1,\ldots, k\}$, denotes the vector $(p_1, p_2, \ldots, p_k)$. Similarly, if each $\*p_i$ is a vector, then $\*P = (\*p_i)_{i\in\mathcal{Z}}$ is a matrix with each vector $\*p_i\tr$ as a row. The expression $(\*p_i)_j \in \Real$ represents the element in $i$-th row and $j$-th column. Calligraphic letters and uppercase Greek letters (such as $\mathcal{X}$ and $\Xi$) are reserved for sets. The symbols $\one$ and $\zero$ denote vectors of all ones and all zeros, respectively, of the size appropriate to their context. The symbol $\eye$ denotes the identity matrix of the appropriate size. The probability simplex in $\RealPlus^\statecount$ is denoted as $\Delta^\statecount = \left\{ \bm{p} \in \RealPlus^S \ss  \one\tr \bm{p} = 1 \right\}$.  The set $\Real$ represents real numbers and the set $\RealPlus$ represents non-negative real numbers.


\section{Related Work} \label{sec:prior_work}

We review relevant prior work that aims at (i) reducing the number of iterations needed to compute an optimal RMDP policy, as well as (ii) reducing the computational complexity of each iteration. We also survey algorithms for related machine learning problems.

The standard approach for computing an optimal RMDP policy is \emph{robust value iteration}, which is a variant of the classical value iteration for non-robust MDPs that iteratively applies the robust Bellman operator to an increasingly accurate approximation of the optimal robust value function~\cite{Givan2000,Iyengar2005,LeTallec2007,Wiesemann2013}. Robust value iteration is easy to implement and versatile, and it converges linearly with a rate of $\gamma$, the discount factor of the RMDP. 

Unfortunately, robust value iteration requires many iterations and thus performs poorly when the discount factor of the RMDP approaches 1. To alleviate this issue, \emph{robust policy iteration} alternates between robust policy evaluation steps that determine the robust value function for a fixed policy and policy improvement steps that select the optimal greedy policy for the current estimate of the robust value function~\cite{Iyengar2005,Hansen2013}. While the theoretical convergence rate guarantee for the robust policy iteration matches that for the robust value iteration, its practical performance tends to be superior for discount factors close to 1. However, unlike the classical policy iteration for non-robust MDPs, which solves a system of linear equations in each policy evaluation step, robust policy iteration solves a large LP in each robust policy evaluation step. This restricts robust policy iteration to small RMDPs.

\emph{Modified policy iteration}, also known as optimistic policy iteration, tends to significantly outperform both value and policy iteration on non-robust MDPs~\cite{Puterman2005}. Modified policy iteration adopts the same strategy as policy iteration, but it merely approximates the value function in each policy evaluation step by executing a small number of value iterations. Generalizing the modified policy iteration to RMDPs is not straightforward. There were several early attempts to develop a robust modified policy iteration~\cite{Satia1973,White1994}, but their convergence guarantees are in doubt~\cite{Kaufman2013}. The challenge is that the alternating maximization (in the policy improvement step) and minimization (in the policy evaluation step) may lead to infinite cycles in the presence of approximation errors. Several natural robust policy iteration variants have been shown to loop infinitely on some inputs~\cite{Condon1993}.

To the best of our knowledge, \emph{robust modified policy iteration}~(RMPI) is the first generalization of the classical modified policy iteration to RMDPs with provable convergence guarantees~\cite{Kaufman2013}. RMPI alternates between robust policy evaluation steps and policy improvement steps. The robust policy evaluation steps approximate the robust value function of a fixed policy by executing a small number of value iterations, and the policy improvement steps select the optimal greedy policy for the current estimate of the robust value function. Our partial policy iteration~(PPI) improves on RMPI in several respects. RMPI only applies to sa-rectangular problems in which there exist optimal deterministic policies, while PPI also applies to s-rectangular problems in which all optimal policies may be randomized. Also, RMPI relies on a value iteration to partially evaluate a fixed policy, whereas PPI can evaluate the fixed policy more efficiently using other schemes such as policy or modified policy iteration. Finally, PPI enjoys a guaranteed linear convergence rate of $\gamma$.

Apart from variants of the robust value and the robust (modified) policy iteration, efforts have been undertaken to efficiently evaluate the robust Bellman operator for structured classes of ambiguity sets. While this evaluation amounts to the solution of a convex optimization problem for generic convex ambiguity sets and reduces to the solution of an LP for polyhedral ambiguity sets, the resulting polynomial runtime guarantees are insufficient due to the large number of evaluations required. Quasi-linear time algorithms for computing Bellman updates for RMDPs with \emph{unweighted} sa-rectangular $L_1$-ambiguity sets have been proposed by~\citet{Iyengar2005} and~\citet{Petrik2014}. Similar algorithms have been used to guide the exploration of MDPs~\cite{Strehl2009,Taleghan2015}. In contrast, our algorithm for sa-rectangular ambiguity sets applies to both unweighted and weighted $L_1$-ambiguity sets, where the latter ones have been shown to provide superior robustness guarantees~\cite{Behzadian2019}. The extension to weighted norms requires a surprisingly large change to the algorithm. Quasi-linear time algorithms have also been proposed for sa-rectangular $L_\infty$-ambiguity sets~\cite{Givan2000}, $L_2$-ambiguity sets~\cite{Iyengar2005} and KL-ambiguity sets~\cite{Iyengar2005,Nilim2005}. We are not aware of any previous specialized algorithms for s-rectangular ambiguity sets, which are significantly more challenging as all optimal policies may be randomized, and it is therefore not possible to compute the worst transition probabilities independently for each action.

Our algorithm for computing the robust Bellman operator over an sa-rectangular ambiguity set resembles LARS, a homotopy method for solving the LASSO problem~\cite{Drori2006,Hastie2009,Murphy2012}. It also resembles methods for computing fast projections onto the $L_1$-ball~\cite{Duchi2008,Thai2015} and the weighted $L_1$-ball~\cite{Berg2011}. In contrast to those works, our algorithm optimizes a linear function (instead of a more general quadratic one) over the intersection of the (weighted) $L_1$-ball and the probability simplex (as opposed to the entire $L_1$-ball).

Our algorithm for computing the robust Bellman operator for s-rectangular ambiguity sets employs a bisection method. This is a common optimization technique for solving low-dimensional problems. We are not aware of works that use bisection to solve s-rectangular RMDPs or similar machine learning problems. However, a bisection method has been previously used to solve sa-rectangular RMDPs with KL-ambiguity sets~\cite{Nilim2005}. That bisection method, however, has a different motivation, solves a different problem, and bisects on different problem parameters.

Throughout this paper, we focus on RMDPs with sa-rectangular or s-rectangular ambiguity sets but note that several more-general classes have been proposed recently~\cite{Mannor2012,Mannor2016,Goyal2018}. These k-rectangular and r-rectangular sets have tangible advantages, but also introduce additional computational complications. 

\section{Robust Markov Decision Processes} \label{sec:robust_bellman_updates}

This section surveys RMDPs and their basic properties. We cover both sa-rectangular and s-rectangular ambiguity sets but limit the discussion to norm-constrained ambiguity sets.

An MDP $(\mathcal{S}, \mathcal{A}, \bm{p}_0, \bm{p}, \bm{r}, \gamma)$ is described by a state set $\states = \{1,\ldots, \statecount \}$ and an action set $\actions = \{ 1, \ldots, \actioncount \}$. The initial state is selected randomly according to the distribution $\b{p}_0\in\Delta^S$. When the MDP is in state $s \in \states$, taking the action $a \in \actions$ results in a stochastic transition to a new state $s' \in \states$ according to the distribution $\bm{p}_{s,a} \in \Delta^\statecount$ with a reward of $r_{s,a,s'} \in \Real$. We condense the transition probabilities $\bm{p}_{s,a}$ to the transition function $\bm{p} = (\bm{p}_{s,a})_{s\in\states,a\in\actions} \in (\Delta^S)^{S \times A}$ which can also be also interpreted as a function $\*p : \states \times\actions \to \Delta^S$. Similarly, we condense the rewards to vectors $\bm{r}_{s,a} = (r_{s,a,s'})_{s'\in\states} \in \Real^S$ and $\bm{r} = (\bm{r}_{s,a})_{s\in\states,a\in\actions}$. The discount factor is $\gamma \in (0, 1)$.

A (stationary) randomized policy $\bm{\pi} = (\bm{\pi}_s)_{s\in\states}$, $\bm{\pi}_s \in \Delta^A$ for all $s \in \states$, is a function that prescribes to take an action $a \in \actions$ with the probability $\pi_{s,a}$ whenever the MDP is in a state $s \in \states$. We use $\Pi = (\Delta^A)^S$ to denote the set of all randomized stationary policies. 

For a given policy $\bm{\pi} \in \Pi$, an MDP becomes a \emph{Markov reward process}, which is a Markov chain with the $S\times S$ transition matrix $\bm{P} (\bm{\pi}) = (\bm{p}_s (\bm{\pi}))_{s\in\states}$ and the rewards $\bm{r} (\bm{\pi}) = (r_s (\bm{\pi}))_{s\in\states} \in \Real^S$ where
\[\b{p}_{s}(\bm{\pi}) = \sum_{a\in\actions} \pi_{s,a} \cdot \b{p}_{s,a}  \quad \text{and} \quad r_s(\bm{\pi}) = \sum_{a\in\actions} \pi_{s,a}\cdot \b{p}_{s,a}\tr \b{r}_{s,a}~, \]
and $\b{p}_{s}(\bm{\pi}) \in \Delta^S$ and $r_s(\*\pi) \in \Real$. The total expected discounted reward of this Markov reward process is
\[\E{\sum_{t = 0}^\infty \gamma^t \cdot r_{S_t, A_t, S_{t+1}} } \;\; = \;\; \b{p}_0\tr (\eye - \gamma\cdot \b{P}(\bm{\pi}))^{-1} \b{r}(\bm{\pi})~. \]
Here, the initial random state $S_0$ is distributed according to $\*p_0$, the subsequent random states $S_1, S_2, \ldots$ are distributed according to $\bm{p} (\bm{\pi})$, and the random actions $A_0, A_1, \ldots$ are distributed according to $\bm{\pi}$. The value function of this Markov reward process is $\bm{v} (\bm{\pi}, \bm{p}) = (\eye - \gamma \cdot \b{P}(\bm{\pi}))^{-1} \b{r}(\bm{\pi})$. For each state $s \in \states$, $v_s (\bm{\pi}, \bm{p})$ describes the total expected discounted reward once the Markov reward process enters $s$. It is well-known that the total expected discounted reward of an MDP is optimized by a deterministic policy $\bm{\pi}$ satisfying $\pi_{s,a} \in \{ 0, 1\}$ for each $s \in \states$ and $a \in \actions$~\cite{Puterman2005}.

RMDPs generalize MDPs in that they account for the uncertainty in the transition function $\bm{p}$. More specifically, the RMDP $(\mathcal{S}, \mathcal{A}, \bm{p}_0, \mathcal{P}, \bm{r}, \gamma)$ assumes that the transition function $\bm{p}$ is chosen adversarially from an \emph{ambiguity set} (or \emph{uncertainty set}) of plausible values $\probs \subseteq (\Delta^S)^{S \times A}$~\cite{Hanasusanto2013,Wiesemann2013,Petrik2014,Petrik2016a,Petrik2019}. The objective is to compute a policy $\bm{\pi} \in \Pi$ that maximizes the \emph{return}, or the expected sum of discounted rewards, under the worst-case transition function from $\probs$:
\begin{equation} \label{eq:the_mother_of_all_problems}
\max_{\bm{\pi}\in\Pi} \min_{\bm{p}\in\probs} \; \b{p}_0\tr \b{v}(\bm{\pi}, \bm{p})~.
\end{equation}
The maximization in~\eqref{eq:the_mother_of_all_problems} represents the objective of the agent, while the minimization can be interpreted as the objective of adversarial nature. To ensure that the minimum exists, we assume throughout the paper that the set $\probs$ is \emph{compact}.
                                                               
The optimal policies in RMDPs are history-dependent, stochastic and NP-hard to compute even when restricted to be stationary~\cite{Iyengar2005,Wiesemann2013}. However, the problem~\eqref{eq:the_mother_of_all_problems} is tractable for some broad classes of ambiguity sets $\probs$. The most common such class are the \emph{sa-rectangular ambiguity sets}, which are defined as Cartesian products of sets $\probs_{s,a} \subseteq \Delta^S$ for each state $s$ and action $a$ \cite{Iyengar2005,Nilim2005,LeTallec2007}: 
\begin{equation} \label{eq:ambiguity_sa_rect}
\probs = \Bigl\{ \bm{p} \in (\Delta^S)^{S \times A} \ss \bm{p}_{s,a} \in \probs_{s,a} \;\; \forall s\in\states, a\in\actions \Bigr\} ~.
\end{equation}
Since each probability vector $\bm{p}_{s,a}$ belongs to a separate set $\probs_{s,a}$, adversarial nature can select the worst transition probabilities independently for each state and action. This amounts to nature being able to observe the agent's action prior to choosing the transition probabilities. Similar to non-robust MDPs, there always exists an optimal deterministic stationary policy in sa-rectangular RMDPs~\cite{Iyengar2005,Nilim2005}.

In this paper, we study sa-rectangular ambiguity sets that constitute weighted $L_1$-balls around some \emph{nominal transition probabilities} $\bar{\b{p}}_{s,a} \in \Delta^S$:
\[ \mathcal{P}_{s,a} = \left\{ \b{p} \in \Delta^\statecount \ss \| \b{p} - \bar{\b{p}}_{s,a} \|_{1,\b{w}_{s,a}} \le \kappa_{s,a} \right\} \]
Here, the weights $\b{w}_{s,a}\in\RealPlus^S$ are assumed to be strictly positive: $w_{s,a} > \zero, s\in\states, a\in\actions$. The radius $\kappa_{s,a} \in \RealPlus$ of the ball is called the \emph{budget}, and the weighted $L_1$-norm is defined as 
\[ \| \b{x} \|_{1,\b{w}} = \sum_{i=1}^{n} w_i \, \lvert x_i \rvert~. \]
Various $L_1$-norm ambiguity sets have been applied to a broad range of  RMDPs~\cite{Iyengar2005,Petrik2014,Petrik2016a,Behzadian2019,Russel2019,Derman2019} and have also been used to guide exploration in MDPs~\cite{Strehl2009,Auer2010,Taleghan2015}.

Similarly to MDPs, the robust value function $\b{v}_{\bm{\pi}} = \min_{\bm{p} \in \mathcal{P}} \bm{v} (\bm{\pi}, \bm{p})$ of an sa-rectangular RMDP for a policy $\bm{\pi} \in \Pi$ can be computed using the \emph{robust Bellman policy update} $\Bell_{\bm{\pi}}: \Real^S \to \Real^S$. For sa-rectangular RMDPs constrained by the $L_1$-norm, the operator $\Bell_{\bm{\pi}}$ is defined for each state $s\in\states$ as
\begin{equation} \label{eq:bellman_sa_rectangular_pol}
\begin{aligned}
    ( \Bell_{\bm{\pi}} \b{v} )_s &= \sum_{a\in\actions}  \left( \pi_{s,a} \cdot \min_{\*p\in\probs_{s,a}} \*p\tr (\*r_{s,a} + \gamma\cdot \*v) \right) \\
    &= \sum_{a\in\actions}  \left( \pi_{s,a} \cdot \min_{\*p\in \Delta^\statecount} \left\{ \*p\tr (\*r_{s,a} + \gamma\cdot \*v) \ss  \| \b{p} - \bar{\b{p}}_{s,a} \|_{1,\b{w}_{s,a}} \le \kappa_{s,a} \right\} \right) ~.
\end{aligned}
\end{equation}
The robust value function is the unique solution to $\b{v}_{\*\pi} = \Bell_{\bm{\pi}} \b{v}_{\*\pi}$~\cite{Iyengar2005}. To compute the optimal value function, we use the sa-rectangular \emph{robust Bellman optimality operator} $\Bell : \Real^S \rightarrow \Real^S$ defined as
\begin{equation} \label{eq:bellman_sa_rectangular}
\begin{aligned}
    (\Bell \b{v})_s &= \max_{a\in\actions} \min_{\*p\in\probs_{s,a}} \*p\tr (\*r_{s,a} + \gamma\cdot \*v) \\
    &= \max_{a\in\actions} \min_{\*p\in \Delta^\statecount} \left\{ \*p\tr (\*r_{s,a} + \gamma \cdot \*v)  \ss \| \b{p} - \bar{\b{p}}_{s,a} \|_{1,\b{w}_{s,a}} \le \kappa_{s,a} \right\}~.
\end{aligned}
\end{equation}
Let $\bm{\pi}^\star\in\Pi$ be an optimal robust policy which solves~\eqref{eq:the_mother_of_all_problems}. Then the optimal robust value function $\*v\opt = \*v_{\*\pi\opt}$ is the unique vector that satisfies $\b{v}\opt = \Bell \b{v}\opt$~\cite{Iyengar2005,Wiesemann2013}.  

Note that the $\bm{p} \in \Delta^S$ in the equations above represents a probability vector rather than the transition function $\bm{p} \in (\Delta^S)^{S \times A}$. To prevent confusion between the two in the remainder of the paper, we specify the dimensions of $\*p$ whenever it is not obvious from its context. 

As mentioned above, sa-rectangular sets assume that nature can observe the agent's action when choosing the robust transition probabilities. This assumption grants nature too much power and  often results in overly conservative policies \cite{LeTallec2007,Wiesemann2013}. \emph{S-rectangular ambiguity sets} partially alleviate this issue while preserving the computational tractability of sa-rectangular sets. They are defined as Cartesian products of sets $\probs_s \subseteq (\Delta^S)^A$ for each state $s$ (as opposed to state-action pairs earlier):
\begin{equation} \label{eq:ambiguity_s_rect}
\probs = \left\{ \bm{p} \in (\Delta^S)^{S \times A} \ss (\bm{p}_{s,a})_{a\in\actions} \in \probs_{s} \; \forall s\in\states \right\}
\end{equation}
Since the probability vectors $\bm{p}_{s,a}$, $a \in \actions$, for the same state $s$ are subjected to the joint constraints captured by $\probs_{s}$, adversarial nature can no longer select the worst transition probabilities independently for each state and action. The presence of these joint constraints amounts to nature choosing the transition probabilities while only observing the state and not the agent's action (but observing the agent's policy). In contrast to non-robust MDPs and sa-rectangular RMDPs, s-rectangular RMDPs are optimized by randomized policies in general \cite{LeTallec2007,Wiesemann2013}. As before, we restrict our attention to s-rectangular ambiguity sets defined in terms of $L_1$-balls around nominal transition probabilities:
\[\mathcal{P}_{s} = \left\{ \bm{p} \in (\Delta^S)^A \ss \sum_{a\in\actions} \lVert \b{p}_a - \bar{\b{p}}_{s,a} \rVert_{1,\b{w}_{s,a}} \le \kappa_{s} \right\}\]
In contrast to the earlier sa-rectangular ambiguity set, nature is now restricted by a single budget $\kappa_s \in \RealPlus$ for all transition probabilities $(\bm{p}_{s,a})_{a\in\actions}$ relating to a state $s \in \states$. We note that although sa-rectangular ambiguity sets are a special case of s-rectangular ambiguity sets in general, this is not true for our particular classes of $L_1$-ball ambiguity sets.

The s-rectangular \emph{robust Bellman policy update} $\Bell_{\*\pi} : \Real^S \to \Real^S$ is defined as
\begin{equation} \label{eq:bellman_s_rectangular_pol}
\begin{aligned}
    ( \Bell_{\bm{\pi}} \b{v} )_s &= \min_{\*p \in \probs_s}  \sum_{a\in\actions} \left( \pi_{s,a} \cdot \*p_a\tr (\*r_{s,a} + \gamma\cdot \*v) \right) \\
    &= \min_{\*p\in (\Delta^\statecount)^A} \left\{ \sum_{a\in\actions} \pi_{s,a} \cdot \*p_a \tr (\*r_{s,a} + \gamma\cdot\*v)  \ss \sum_{a\in\actions} \lVert \b{p}_a - \bar{\b{p}}_{s,a} \rVert_{1,\b{w}_{s,a}} \le \kappa_{s} \right\}~.
\end{aligned}
\end{equation}
As in the sa-rectangular case, the robust value function is the unique solution to $\b{v}_{\*\pi} = \Bell_{\bm{\pi}} \b{v}_{\*\pi}$ \cite{Wiesemann2013}. The s-rectangular \emph{robust Bellman optimality operator} $\Bell : \Real^S \to \Real^S$ is defined as
\begin{equation} \label{eq:bellman_s_rectangular}
\begin{aligned}
(\Bell \b{v})_s &= \max_{\bm{d} \in\Delta^A}  \min_{\*p\in\probs_{s}} \sum_{a\in\actions} d_a \cdot \*p_a\tr (\*r_{s,a} + \gamma\cdot \*v) \\
&= \max_{\bm{d} \in\Delta^A}  \min_{\*p\in (\Delta^\statecount)^A} \left\{ \sum_{a\in\actions} d_a \cdot \*p_a \tr (\*r_{s,a} + \gamma\cdot\*v)  \ss \sum_{a\in\actions} \lVert \b{p}_a - \bar{\b{p}}_{s,a} \rVert_{1,\b{w}_{s,a}} \le \kappa_{s} \right\}~.
\end{aligned}
\end{equation}
The optimal robust value function $\*v\opt = \*v_{\*\pi\opt}$ in an s-rectangular RMDP is also the unique vector that satisfies $\b{v}\opt = \Bell \b{v}\opt$~\cite{Iyengar2005,Wiesemann2013}. We use the same symbols $\Bell_{\bm{\pi}}$ and $\Bell$ for sa-rectangular and s-rectangular ambiguity sets; their meaning will be clear from the context.

\section{Partial Policy Iteration} \label{sec:ppi}

In this section, we describe and analyze a new iterative method for solving RMDPs with sa-rectangular or s-rectangular ambiguity sets which we call \emph{Partial Policy Iteration}~(PPI). It resembles standard policy iteration; it evaluates policies only partially before improving them. PPI is the first policy iteration method that provably converges to the optimal solution for s-rectangular RMDPs. We first describe and analyze PPI and then compare it with existing robust policy iteration algorithms.

\begin{algorithm}[tb]
    \KwIn{Tolerances $\epsilon_1, \epsilon_2, \ldots$ such that $\epsilon_{k+1} < \gamma \epsilon_k$ and desired precision $\delta$}
    \KwOut{Policy $\bm{\pi}_k$ such that $\norm{\bm{v}_{\bm{\pi}_k}  - \*v\opt}_\infty \le \delta$}
    $k \gets 0$,~
    $\*v_0 \gets $ an arbitrary initial value function \;
    \Repeat{$\norm{\Bell \b{v}_k - \b{v}_k}_\infty < \frac{1-\gamma}{2}\, \delta $}{
        $k\gets k+1$\;
        \tcp{Policy improvement}
        Compute $\b{\tilde v}_k \gets \Bell \b{v}_{k-1}$ and choose \emph{greedy} $\bm{\pi}_k$ such that $\Bell_{\bm{\pi}_k} \b{v}_{k-1} = \b{\tilde v}_k$\;
        \tcp{Policy evaluation}
        Solve MDP in Def.~\ref{def:evaluation_mdp} to get $\b{v}_k$ such that $\norm{\Bell_{\bm{\pi}_{k}} \b{v}_k - \b{v}_k}_\infty \le (1-\gamma)\,\epsilon_k$ \;
    }
    \Return $\bm{\pi}_k$
    \caption{Partial Policy Iteration~(PPI)} 
    \label{alg:rpi}
\end{algorithm}

\cref{alg:rpi} provides an outline of PPI. The algorithm follows the familiar pattern of interleaving approximate policy evaluation with policy improvement and thus resembles the modified policy iteration (also known as optimistic policy iteration) for classical, non-robust MDPs~\cite{Bertsekas1996a,Puterman2005}. In contrast to classical policy iteration, which always evaluates incumbent policies precisely, PPI approximates policy evaluation. This is fast and sufficient, particularly when evaluating highly suboptimal policies.

Notice that by employing the robust Bellman optimality operator $\Bell$, the policy improvement step in \cref{alg:rpi} selects the updated greedy policy $\bm{\pi}_k$ in view of the worst transition function from the ambiguity set. Although the robust Bellman optimality operator $\Bell$ requires more computational effort than its non-robust counterpart, it is necessary as several variants of PPI that employ a non-robust Bellman optimality operator have been shown to fail to converge to the optimal solution~\cite{Condon1993}.

The policy evaluation step in \cref{alg:rpi} is performed by approximately solving a \emph{robust policy evaluation MDP} defined as follows.
\begin{defn}\label{def:evaluation_mdp}
    For an s-rectangular RMDP $(\mathcal{S}, \mathcal{A}, \bm{p}_0, \mathcal{P}, \bm{r}, \gamma)$ and a fixed policy $\bm{\pi} \in \Pi$, we define the \emph{robust policy evaluation MDP} $(\mathcal{S}, \bar{\mathcal{A}}, \bm{p}_0, \bar{\bm{p}}, \bar{\bm{r}}, \gamma)$ as follows. The continuous state-dependent action sets $\bar\actions(s)$, $s\in\states$, represent nature's choice of the transition probabilities and are defined as $\bar\actions(s) = \probs_s$. Thus, nature's decisions are of the form $\bm{\alpha} = (\bm{\alpha}_a)_{a\in\actions} \in (\Delta^S)^A$ with $\bm{\alpha}_a \in \Delta^S$, $a \in \mathcal{A}$. The transition function $\bar{\*p}$ and the rewards $\bar{\bm{r}}$ are defined as
    \[ \bar{\*p}_{s, \bm{\alpha}} = \sum_{a\in\actions} \pi_{s,a} \cdot \*\alpha_a 
    \quad \text{and} \quad
    \bar{r}_{s, \bm{\alpha}} = - \sum_{a\in\actions} \pi_{s,a}\cdot \b{\alpha}_a\tr \b{r}_{s,a}~ , \]    
    where $\bar{\*p}_{s, \bm{\alpha}} \in \Delta^S$ and $\bar{r}_{s, \bm{\alpha}} \in \Real$.     All other parameters of the robust policy evaluation MDP coincide with those of the RMDP. Moreover, for sa-rectangular RMDPs we replace $\bar\actions(s) = \probs_s$ with $\bar\actions(s) = \times_{a\in\actions} \probs_{s,a}$.
\end{defn}

We emphasize that although the robust policy evaluation MDP in \cref{def:evaluation_mdp} computes the robust value function of the policy $\bm{\pi}$, it is, nevertheless a regular non-robust MDP. Indeed, although the robust policy evaluation MDP has an infinite action space, its optimal value function exists since the Assumptions 6.0.1--6.0.4 of \citet{Puterman2005} are satisfied. Moreover, since the rewards $\bar{\bm{r}}$ are continuous (in fact, linear) in $\bm{\alpha}$ and the sets $\bar{\mathcal{A}}(s)$ are compact by construction of $\probs$, there also exists an optimal deterministic stationary policy by Theorem 6.2.7 of \citet{Puterman2005} and the extreme value theorem. When the action sets $\bar{\mathcal{A}}(s)$ are polyhedral, the greedy action for each state can be computed readily from an LP, and the MDP can be solved using any standard MDP algorithm. \cref{sec:value_function_update} describes a new algorithm that computes greedy actions in quasi-linear time, which is much faster than the time required by generic LP solvers.

The next proposition shows that the optimal solution to the robust policy evaluation MDP from \cref{def:evaluation_mdp} indeed corresponds to the robust value function $\bm{v}_{\bm{\pi}}$ of the policy $\bm{\pi}$. 

\begin{prop}\label{prop:mdp_rmdp_equivalence}
    For an RMDP $(\mathcal{S}, \mathcal{A}, \bm{p}_0, \mathcal{P}, \bm{r}, \gamma)$ and a policy $\bm{\pi} \in \Pi$, the optimal value function $\bar{\*v}\opt$ of the associated robust policy evaluation MDP satisfies $\bar{\*v}\opt = -\*v_{\bm{\pi}}$.
\end{prop}

\begin{proof}
    Let $\bar\Bell$ be the Bellman operator for the robust policy evaluation MDP. To prove the result, we first argue that $\bar{\Bell} \*v = - (\Bell_{\bm{\pi}} (-\*v))$ for every $\*v\in\Real^S$. Indeed, \cref{def:evaluation_mdp} and basic algebraic manipulations reveal that
    \[
        \begin{array}{ll@{}l@{\quad}l} 
            \displaystyle (\bar{\Bell} \*v)_s
            &=&
            \displaystyle \max_{\bm{\alpha} \in \bar{\actions}(s)} \; \bar{r}_{s,\bm{\alpha}} + \gamma \cdot \bar{\*p}_{s,\bm{\alpha}}\tr \*v \\
            &=&
            \displaystyle \max_{\bm{\alpha} \in \probs_s} \; \left( - \sum_{a\in\actions} \pi_{s,a}\cdot \b{\alpha}_a\tr \b{r}_{s,a} \right) + \gamma \cdot  \left(\sum_{a\in\actions} \pi_{s,a} \cdot \*\alpha_a \right)\tr \*v & \text{(from \cref{def:evaluation_mdp})} \\
            &=&
            \displaystyle \max_{\bm{\alpha} \in \probs_s} \; \sum_{a\in\actions} \pi_{s,a}\cdot \*\alpha_a \tr \left( - \*r_{s,a} + \gamma \cdot \*v \right) \\
            &=-&
            \displaystyle \min_{\bm{\alpha} \in \probs_s} \; \sum_{a\in\actions} \pi_{s,a}\cdot \*\alpha_a \tr \left( \*r_{s,a} + \gamma \cdot (-\*v) \right) \;\;=\;\; (- \Bell_{\bm{\pi}} (-\*v) )_s~.
        \end{array}
    \]
    Let $\bar{\*v}\opt = \bar{\Bell} \bar{\*v}\opt$ be the fixed point of $\bar{\Bell}$, whose existence and uniqueness is guaranteed by the Banach fixed-point theorem since $\bar\Bell$ is a contraction under the $L_\infty$-norm. Substituting $\bar{\*v}\opt$ into the identity above then gives
    \[ 
    \bar{\*v}\opt = \bar{\Bell} \bar{\*v}\opt = -\Bell_{\bm{\pi}} (-\bar{\*v}\opt)
    \quad \Longrightarrow \quad
    -\bar{\*v}\opt = \Bell_{\bm{\pi}} (- \bar{\*v}\opt)~, 
    \]
    which shows that $-\bar{\*v}\opt$ is the unique fixed point of $\Bell_{\bm{\pi}}$ since this operator is also an $L_\infty$-contraction (see~\cref{lem:contraction} in \cref{{sec:bell_props}}).
\end{proof}

The robust policy evaluation MDP can be solved by value iteration, (modified) policy iteration, linear programming, or another suitable method. We describe in~\cref{sec:value_function_update} an efficient algorithm for calculating $\Bell_{\bm{\pi}_k}$. The accuracy requirement $\norm{\Bell_{\bm{\pi}_k} \b{v}_k - \b{v}_k}_\infty \le (1-\gamma)\,\epsilon_k$ in \cref{alg:rpi} can be used as the stopping criterion in the employed method. As we show next, this condition guarantees that $\norm{\*v_k - \*v_{\bm{\pi}_k}}_\infty \le \epsilon_k$, that is, $\bm{v}_k$ is an $\epsilon_k$-approximation to the robust value function of $\bm{\pi}_k$.

\begin{prop} \label{cor:pol_value_bound}
    Consider any value function $\*v_k$ and any policy $\bm{\pi}_k$ greedy for $\*v_k$, that is, $\Bell_{\bm{\pi}_k} \*v_k = \Bell \*v_k$. The robust value function $v_{\bm{\pi}_k}$ of  $\bm{\pi}_k$ can then be bounded as follows.
    \[ \norm{\*v_{\bm{\pi}_k} - \*v_k}_\infty \le \frac{1}{1-\gamma} \norm{\Bell_{\bm{\pi}_k} \*v_k - \*v_k}_\infty\]
\end{prop}

\begin{proof}
    The statement follows immediately from \cref{lem:opt_value_function_approx} in \cref{sec:bell_props} if we set $\bm{\pi} = \bm{\pi}_k$ and $\*v = \*v_k$.
\end{proof}

\cref{alg:rpi} terminates once the condition $\norm{\Bell \b{v}_k - \b{v}_k}_\infty < \frac{1-\gamma}{2}\, \delta$ is met. Note that this condition can be verified using the computations from the current iteration and thus does not require a new application of the Bellman optimality operator. As the next proposition shows, this termination criterion guarantees that the computed policy $\bm{\pi}_k$ is within $\delta$ of the optimal policy.

\begin{prop}\label{cor:opt_value_bound}
    Consider any value function $\*v_k$ and any policy $\bm{\pi}_k$ greedy for $\*v_k$. If $\*v\opt$ is the optimal robust value function, then
    \[ \norm{\*v\opt - \*v_{\bm{\pi}_k}}_\infty \le \frac{2}{1-\gamma} \norm{\Bell \*v_k - \*v_k}_\infty ~,\]
    where $\*v_{\bm{\pi}_k}$ the robust value function of $\bm{\pi}_k$. 
\end{prop}

The statement of \cref{cor:opt_value_bound} parallels the well-known properties of approximate value functions for classical, non-robust MDPs~\cite{Williams1993a}.

\begin{proof}[Proof of \cref{cor:opt_value_bound}]
    Using the triangle inequality of vector norms, we see that
    \[ 
    \norm{\*v\opt - \*v_{\bm{\pi}_k}}_\infty \le \norm{\*v\opt  - \*v_k}_\infty + \norm{\*v_k - \*v_{\bm{\pi}_k}}_\infty~.
    \]
    Using \cref{lem:opt_value_function_approx} in \cref{sec:bell_props} with $\*v = \*v_k$, the first term $\norm{\*v\opt  - \*v_k}_\infty$ can be bounded from above as follows.
    \[ 
    \norm{\*v\opt  - \*v_k}_\infty \le \frac{1}{1-\gamma} \norm{\Bell \*v_k - \*v_k}_\infty
    \]
    The second term $\norm{\*v_k - \*v_{\bm{\pi}_k}}_\infty$ above can be bounded using \cref{cor:pol_value_bound} and the fact that $\Bell_{\bm{\pi}_k} \*v_k = \Bell \*v_k$, which holds since $\bm{\pi}_k$ is greedy for $\*v_k$:
    \[
    \norm{\*v_k - \*v_{\bm{\pi}_k}}_\infty \le \frac{1}{1-\gamma} \norm{\Bell \*v_k - \*v_k}_\infty
    \]
    The result then follows by combining the two bounds.
\end{proof}

We are now ready to show that PPI converges linearly with a rate of at most $\gamma$ to the optimal robust value function. This is no worse than the convergence rate of the robust value iteration. The result mirrors similar results for classical, non-robust MDPs. Regular policy iteration is not known to converge at a faster rate than value iteration even though it is strongly polynomial~\cite{Puterman2005,Post2015,Hansen2013}.

\begin{thm}\label{thm:ppi_convergence}
    Consider $c > 1$ such that $\epsilon_{k+1} \le \gamma^c \, \epsilon_k$ for all $k$ in  \cref{alg:rpi}. Then the optimality gap of the policy $\bm{\pi}_{k+1}$ computed in each iteration $k \ge 1$ is bounded from above by
    \[ \norm{\b{v}\opt - \b{v}_{\bm{\pi}_{k+1}}}_\infty \le \gamma^k \left(\norm{\b{v}\opt - \b{v}_{\bm{\pi}_1}}_\infty + \frac{2 \, \epsilon_1}{(1-\gamma^{c-1})(1-\gamma)} \right)~.\]
\end{thm}

\cref{thm:ppi_convergence} requires the sequence of acceptable evaluation errors $\epsilon_k$ to decrease faster than the discount factor $\gamma$. As one would expect, the theorem shows that smaller values of $\epsilon_k$ lead to a faster convergence in terms of the number of iterations. On the other hand, smaller $\epsilon_k$ values also imply that each individual iteration is computationally more expensive.

The proof of \cref{thm:ppi_convergence} follows an approach similar to the convergence proofs  of policy iteration~\cite{Puterman1979a,Puterman2005}, modified policy iteration~\cite{Puterman1978,Puterman2005} and robust modified policy iteration~\cite{Kaufman2013}. The proofs for (modified) policy iteration start by assuming that the initial value function $\*v_0$ satisfies $\*v_0 \le \*v\opt$; the policy updates and evaluations then increase $\*v_k$ as fast as value iteration while preserving $\*v_k \le \*w_k$ for some $\*w_k$ satisfying $\lim_{k=\infty} \*w_k = \*v\opt$. The incomplete policy evaluation in RMDPs may result in $\*v_k\ge \*v\opt$, which precludes the use of the modified policy iteration proof strategy. The convergence proof for RMPI inverts the argument by starting with $\*v_0 \ge \*v\opt$ and decreasing $\*v_k$ while preserving $\*v_k \ge \*w_k$. This property, however, is only guaranteed to hold when the policy evaluation step is performed using value iteration. PPI, on the other hand, makes no assumptions on how the policy evaluation step is performed. Its approximate value functions $\*v_k$ may not satisfy $\*v_k \le \*v\opt$, and the decreasing approximation errors $\epsilon_{k}$ guarantee improvements in $\*v_{\bm{\pi}_k}$ that are sufficiently close to those of robust policy iteration. A key challenge is that $\*v_k \neq \*v_{\bm{\pi}_k}$, which implies that the incumbent policies $\bm{\pi}_k$ can actually become \emph{worse} in the short run.

\begin{proof}[Proof of \cref{thm:ppi_convergence}]
    We first show that the robust value function of policy $\bm{\pi}_{k+1}$ is at least as good as that of $\bm{\pi}_k$ with a tolerance that depends on $\epsilon_k$. Using this result, we then prove that in each iteration $k$, the optimality gap of the determined policy $\bm{\pi}_k$ shrinks by the factor $\gamma$, again with a tolerance that depends on $\epsilon_k$. In the third and final step, we recursively apply our bound on the optimality gap of the policies $\bm{\pi}_1, \bm{\pi}_2, \ldots$ to obtain the stated convergence rate.
    
    We remind the reader that for each iteration $k$ of \cref{alg:rpi}, $\*v_k$ denotes the approximate robust value function of the incumbent policy $\bm{\pi}_k$, whereas $\bm{v}_{\bm{\pi}_k}$ denotes the \emph{precise} robust value function of $\bm{\pi}_k$. We abbreviate the robust Bellman policy update $\Bell_{\bm{\pi}_k}$ by $\Bell_k$. Moreover, we denote by $\bm{\pi}\opt$ the optimal policy with robust value function $\bm{v}\opt$. The proof uses several properties of robust Bellman operators that are summarized in \cref{sec:bell_props}.

    As for the first step, recall that the policy evaluation step of PPI computes a value function $\*v_k$ that approximates the robust value function $\*v_{\bm{\pi}_k}$ within a certain tolerance:
    \[ \norm{\Bell_{k} \b{v}_k - \b{v}_k}_\infty \le (1-\gamma)\,\epsilon_k~.\] 
    Combining this bound with \cref{cor:pol_value_bound} yields $\norm{\*v_{\bm{\pi}_k} - \b{v}_k}_\infty \le \epsilon_k$, which is equivalent to
    \begin{align}
        \label{eq:impr_bound_one} \*v_{\bm{\pi}_k} &\ge \b{v}_k \mspace{10mu} -\epsilon_k \cdot \one \\
        \label{eq:impr_bound_two} \b{v}_k &\ge \*v_{\bm{\pi}_k} - \epsilon_k \cdot \one ~.
    \end{align}
    We use this bound to bound $\Bell_{k+1} \*v_{\bm{\pi}_k}$ from below as follows:
    \begin{equation} \label{eq:lower_bound_bellman}
        \begin{aligned}
            \Bell_{k+1} \*v_{\bm{\pi}_k} &\ge \Bell_{k+1} (\b{v}_k - \epsilon_k \one) && \text{from \eqref{eq:impr_bound_two} and \cref{lem:monotone}} \\
            &\ge \Bell_{k+1} \b{v}_k - \gamma \epsilon_k \one &&
            \text{from \cref{lem:bellman_linear_translation}} \\
            &\ge \Bell_{k} \b{v}_k - \gamma \epsilon_k \one &&
            \text{$\Bell_{k+1}$ is greedy to $\*v_k$} \\
            &\ge \Bell_{k} (\*v_{\bm{\pi}_k} - \epsilon_k \one)  - \gamma \epsilon_k \one &&
            \text{from \eqref{eq:impr_bound_one} and \cref{lem:monotone}} \\
            &\ge \Bell_{k} \*v_{\bm{\pi}_k} - 2 \gamma \epsilon_k \one &&
            \text{from \cref{lem:bellman_linear_translation}}\\
            &\ge \*v_{\bm{\pi}_k} - 2 \gamma \epsilon_k \one &&
            \text{because $\*v_{\bm{\pi}_k} = \Bell_k \*v_{\bm{\pi}_k}$}
        \end{aligned}
    \end{equation}
    This lower bound on $\Bell_{k+1} \*v_{\bm{\pi}_k}$ readily translates into the following lower bound on $\*v_{\bm{\pi}_{k+1}}$:
    \begin{align*}
        \b{v}_{\bm{\pi}_{k+1}} - \b{v}_{\bm{\pi}_{k}} &= \Bell_{k+1} \b{v}_{\bm{\pi}_{k+1}} - \b{v}_{\bm{\pi}_{k}} && \text{from $\*v_{\bm{\pi}_{k+1}} = \Bell_{k+1} \*v_{\bm{\pi}_{k+1}}$} \\
        &= (\Bell_{k+1} \b{v}_{\bm{\pi}_{k+1}} - \Bell_{k+1} \b{v}_{\bm{\pi}_{k}}) + (\Bell_{k+1} \b{v}_{\bm{\pi}_{k}} - \b{v}_{\bm{\pi}_{k}}) && \text{add $0$} \\
        &\ge \gamma \*P (\b{v}_{\bm{\pi}_{k+1}} - \b{v}_{\bm{\pi}_{k}}) + (\Bell_{k+1} \b{v}_{\bm{\pi}_{k}} - \b{v}_{\bm{\pi}_{k}}) &&
        \text{from \cref{lem:Bellman_bound_linear}} \\
        &\ge \gamma \*P (\b{v}_{\bm{\pi}_{k+1}} - \b{v}_{\bm{\pi}_{k}}) - 2 \gamma \epsilon_k \one
        && \text{from \eqref{eq:lower_bound_bellman}}
    \end{align*}
    Here, $\*P$ is the stochastic matrix defined in \cref{lem:Bellman_bound_linear}. Basic algebraic manipulations show that the inequality above further simplifies to
    \[ (\eye - \gamma \*P) (\b{v}_{\bm{\pi}_{k+1}} - \b{v}_{\bm{\pi}_{k}}) \ge -2 \gamma \epsilon_k \one~. \]
    Recall that for any stochastic matrix $\*P$, the inverse $(\eye - \gamma \*P)^{-1}$ exists, is monotone, and satisfies $(\eye - \gamma \*P)^{-1} \one = (1-\gamma)^{-1} \one$, which can all be seen from its von Neumann series expansion. Using these properties, the lower bound on $\*v_{\bm{\pi}_{k+1}}$ simplifies to
    \begin{equation} \label{eq:improvement_bellman}
        \b{v}_{\bm{\pi}_{k+1}} \ge \*v_{\bm{\pi}_k} - \frac{2\, \gamma\, \epsilon_k }{1-\gamma} \one ~,
    \end{equation}
    which concludes the first step.
    
    To prove the second step, note that the policy improvement step of PPI reduces the optimality gap of policy $\bm{\pi}_k$ as follows:
    \begin{align*}
        \b{v}\opt - \b{v}_{\bm{\pi}_{k+1}} &= \b{v}\opt - \Bell_{k+1}\b{v}_{\bm{\pi}_{k+1}} &&
        \text{from the definition of $v_{\bm{\pi}_{k+1}}$} \\
        &= (\b{v}\opt - \Bell_{k+1}\b{v}_{\bm{\pi}_{k}}) - (\Bell_{k+1} \b{v}_{\bm{\pi}_{k+1}} - \Bell_{k+1} \*v_{\bm{\pi}_k}) &&
        \text{subtract $0$} \\
        &\leq (\b{v}\opt - \Bell_{k+1}\b{v}_{\bm{\pi}_{k}}) - \gamma \cdot \*P(\b{v}_{\bm{\pi}_{k+1}} - \*v_{\bm{\pi}_k})  &&
        \text{for some $\*P$ from \cref{lem:Bellman_bound_linear}}\\
        &\le (\b{v}\opt - \Bell_{k+1}\b{v}_{\bm{\pi}_{k}})  +  \frac{2 \gamma^2 \epsilon_k}{1-\gamma} \one 
        && \text{from \eqref{eq:improvement_bellman} and $\*P\one = \one$} \\
        &\le (\b{v}\opt - \Bell_{k+1} \b{v}_{k}) + \left(  \gamma \epsilon_k +  \frac{2 \gamma^2 \epsilon_k }{1-\gamma} \right) \one 
        && \text{from \eqref{eq:lower_bound_bellman}}\\
        &\le (\b{v}\opt - \Bell_{\bm{\pi}\opt} \b{v}_{k}) + \left(  \gamma \epsilon_k +  \frac{2 \gamma^2 \epsilon_k }{1-\gamma} \right) \one 
        && \text{$\Bell_{k+1}$ is greedy to $\*v_k$}\\
        &\le (\b{v}\opt - \Bell_{\bm{\pi}\opt} \b{v}_{\bm{\pi}_k}) + \left(  2 \gamma \epsilon_k +  \frac{2 \gamma^2 \epsilon_k }{1-\gamma} \right) \one 
        && \text{from \eqref{eq:impr_bound_two}}\\
        &= (\Bell_{\bm{\pi}\opt} \b{v}\opt - \Bell_{\bm{\pi}\opt} \b{v}_{\bm{\pi}_{k}}) +  \frac{2 \gamma \epsilon_k}{ 1-\gamma} \one 
        && \text{from $\*v\opt = \Bell_{\bm{\pi}\opt}\*v\opt$}
    \end{align*}
    \cref{cor:optimal_dominates} shows that $\b{v}\opt \geq \b{v}_{\bm{\pi}_{k+1}}$, which allows us to apply the $L_\infty$-norm operator on both sides of the inequality above. Using the contraction property of the robust Bellman policy update (see~\cref{lem:contraction}), the bound above implies that
    \begin{equation} \label{eq:recursive_error}
        \norm{\b{v}\opt - \b{v}_{\bm{\pi}_{k+1}}}_\infty
        \;\; \le \;\;
        \norm{\Bell_{\bm{\pi}\opt} \b{v}\opt - \Bell_{\bm{\pi}\opt} \b{v}_{\bm{\pi}_{k}}}_\infty + \frac{2\gamma\epsilon_k}{1-\gamma}
        \;\; \le \;\;
        \gamma \norm{\b{v}\opt -  \b{v}_{\bm{\pi}_{k}}}_\infty + \frac{2\gamma\epsilon_k}{1-\gamma} ~,
    \end{equation}
    which concludes the second step.

    To prove the second step, we recursively apply the inequality \eqref{eq:recursive_error} to bound the overall optimality gap of policy $\bm{\pi}_{k+1}$ as follows:
    \begin{align*}
        \norm{\b{v}\opt - \b{v}_{\bm{\pi}_{k+1}}}_\infty &\le \gamma \norm{\b{v}\opt - \b{v}_{\bm{\pi}_k}}_\infty + \frac{2\gamma\epsilon_k}{1-\gamma} \\
        &\le \gamma^2 \norm{\b{v}\opt - \b{v}_{\bm{\pi}_{k-1}}}_\infty + \frac{2\gamma\epsilon_k} {1-\gamma}  + \frac{2\gamma^2\epsilon_{k-1} }{1-\gamma}  \\
        &\le \ldots  \\
        &\le \gamma^k \norm{\b{v}\opt - \b{v}_{\bm{\pi}_1}}_\infty + \frac{2}{1-\gamma} \sum_{j=0}^{k-1} \epsilon_{j+1} \gamma^{k-j}  ~.
    \end{align*}
    The postulated choice $\epsilon_j \leq \gamma^c \epsilon_{j-1} \leq \gamma^{2c} \epsilon_{j-2} \leq \ldots \leq \gamma^{(j-1)c} \epsilon_1$ with $c > 1$ implies that
    \[
    \sum_{j=0}^{k-1} \epsilon_{j+1} \gamma^{k-j}
    \;\; \le \;\;
    \epsilon_1  \sum_{j=0}^{k-1} \gamma^{j c} \gamma^{k-j}
    \;\; = \;\;
    \gamma^{k} \epsilon_1 \sum_{j=0}^{k-1} \gamma^{j (c-1)}
    \;\; \le \;\;
    \gamma^k \frac{\epsilon_1}{1-\gamma^{c-1}} ~.
    \]
    The result follows by substituting the value of the geometric series in the bound above.
\end{proof}

PPI improves on several existing algorithms for RMDPs. To the best of our knowledge, the only method that has been shown to solve s-rectangular RMDPs is the robust value iteration~\cite{Wiesemann2013}. Robust value iteration is simple and versatile, but it may be inefficient because it employs the computationally intensive robust Bellman optimality operator $\Bell$ both to evaluate and to improve the incumbent policy. In contrast, PPI only relies on $\Bell$ to improve the incumbent policy $\bm{\pi}_k$, whereas the robust value function of $\bm{\pi}_k$ is evaluated (approximately) using the more efficient robust Bellman policy update $\Bell_{\bm{\pi}_k}$. In addition to robust value iteration, several methods proposed for sa-rectangular RMDPs can potentially be generalized to s-rectangular problems.

Robust Modified Policy Iteration~(RMPI)~\cite{Kaufman2013} is the algorithm for sa-rectangular RMDPs that is most similar to PPI. RMPI can be cast as a special case of PPI in which the policy evaluation step is solved by value iteration rather than by an arbitrary MDP solver. Value iteration can be significantly slower than (modified) policy iteration in this context due to the complexity of computing $\Bell_{\bm{\pi}_k}$. RMPI also does not reduce the approximation error $\epsilon_k$ in the policy evaluations but instead runs a fixed number of value iterations. The decreasing tolerances $\epsilon_k$ of PPI are key to guaranteeing its convergence rate; a comparable convergence rate is not known for RMPI.

Robust policy iteration~\cite{Iyengar2005,Hansen2013} is also similar to PPI, but it has only been proposed in the context of sa-rectangular RMDPs. The main difference to PPI is that the policy evaluation step in robust policy iteration is performed exactly with the tolerance $\epsilon_k = 0$ for all iterations $k$, which can be done by solving a large LP~\cite{Iyengar2005}. Although this approach is elegant and simple to implement, our experimental results show that it does not scale to even moderately-sized problems. 

PPI is general and works for sa-rectangular and s-rectangular RMDPs whose robust Bellman operators $\Bell$ and $\Bell_{\bm{\pi}}$ can be computed efficiently. In the next two sections we show that, in fact, the robust Bellman optimality and update operators can be computed efficiently for sa-rectangular and s-rectangular ambiguity sets defined by bounds on the $L_1$-norm.

\section{Computing the Bellman Operator: SA-Rectangular Sets} \label{sec:fast_solution}

In this section, we develop an efficient homotopy algorithm to compute the sa-rectangular robust Bellman optimality operator $\Bell$ defined in~\eqref{eq:bellman_sa_rectangular}. Our algorithm computes the inner minimization over $\bm{p} \in \mathcal{P}_{s,a}$ in~\eqref{eq:bellman_sa_rectangular}; to compute $\Bell \*v$ for some $\*v\in\Real^S$, we simply execute our algorithm for each action $a \in \actions$ and select the maximum of the obtained objective values. To simplify the notation, we fix a state $s \in \states$ and an action $a \in \actions$ throughout this section and drop the associated subscripts whenever the context is unambiguous (for example, we use $\bar{\*p}$ instead of $\bar{\*p}_{s,a}$). We also fix a value function $\*v$ throughout this section.

Our algorithm uses the idea of homotopy continuation~\cite{Vanderbei2001} to solve the following parametric optimization problem $q: \RealPlus\to\Real$, which is parameterized by $\xi$:
\begin{equation} \label{eq:q_optimization_l1}
	q(\xi) = \min_{\*p \in\Delta^\statecount} \Bigl\{  \b{p}\tr \b{z} \ss \norm{\b{p} - \bar{\b{p}}}_{1,\b{w}} \le \xi \Bigr\}
\end{equation}
Here,  we use the abbreviation $\*z = \b{r}_{s,a} + \gamma\cdot \*v$. Note that $\xi$ plays the role of the budget $\kappa_{s,a}$ in our sa-rectangular uncertainty set $\mathcal{P}_{s,a}$, and that $q (\kappa_{s,a})$ computes the inner minimization over $\bm{p} \in \mathcal{P}_{s,a}$ in~\eqref{eq:bellman_sa_rectangular}. Our homotopy method achieves its efficiency by computing $q(\xi)$ for $\xi = 0$ and subsequently for all $\xi \in (0, \kappa_{s,a}]$ instead of computing $q (\kappa_{s,a})$ directly~\cite{Asif2009,Garrigues2009}. The problem $q (0)$ is easy since the only feasible solution is $\b{p} = \bar{\b{p}}$, and thus $q(0) = \bar{\*p}\tr \*z$. We then trace an optimal solution $\*p\opt(\xi)$ as $\xi$ increases, until we reach $\xi = \kappa_{s,a}$. Our homotopy algorithm is fast because the optimal solution can be traced efficiently when $\xi$ is increased. As we show below, $q(\xi)$ is piecewise affine with at most $S^2$ pieces (or $S$ pieces, if all components of $\*w$ are equal), and exactly two elements of $\*p\opt (\xi)$ change when $\xi$ increases.

By construction, $q(\xi)$ varies with $\xi$ only when $\xi$ is small enough so that the constraint $\norm{\b{p} - \bar{\b{p}}}_{1,\b{w}} \le \xi$ in~\eqref{eq:q_optimization_l1} is binding at optimality. To avoid case distinctions for the trivial case when $\norm{\b{p} - \bar{\b{p}}}_{1,\b{w}} < \xi$ at optimality and $q(\xi)$ is constant, we assume in the remainder of this section that $\xi$ is small enough. Our homotopy algorithm treats large $\xi$ identically to the largest $\xi$ for which the constraint is binding at optimality.

In the remainder of this section, we first investigate the structure of basic feasible solutions to the problem~\eqref{eq:q_optimization_l1} in \cref{sec:param}. We then exploit this structure to develop our homotopy method in \cref{sec:homo}, and we conclude with a complexity analysis in \cref{sec:homotopy_complexity}.

\subsection{Properties of the Parametric Optimization Problem $q (\xi)$}\label{sec:param}

Our homotopy method employs the following LP formulation of problem~\eqref{eq:q_optimization_l1}:
\begin{equation}\label{eq:q_linear_program}
	\begin{array}{r@{\;\,}l@{\quad}l}
		q(\xi) = & \displaystyle \min_{\b{p}, \b{l} \in\Real^S} & \displaystyle \b{z}\tr \b{p} \\
		& \displaystyle \text{subject to}
		& \displaystyle \b{p} - \bar{\b{p}} \le \b{l}  \\
		& & \displaystyle \bar{\b{p}} - \b{p} \le \b{l}  \\
		& & \displaystyle \b{p} \ge \zero \\
		& & \displaystyle \one\tr \b{p} = 1, \;\; \b{w}\tr \b{l} = \xi
	\end{array}
\end{equation}
Note that $\*l\ge\zero$ is enforced implicitly. The standard approach is to solve~\eqref{eq:q_linear_program} using a generic LP algorithm. This is, unfortunately, too slow to be practical as our empirical results show.

\begin{table}
	\centering
	\begin{tabular}{|l|ccccccc|}
		\toprule
		$i\in \ldots \rightarrow$ & $\sN_B$ & $\sU_B$ & $\sL_B$ & $\sUL_B$ & $\sZ_B$ & $\sUZ_B$ & $\sLZ_B$ \\
		\midrule
		$p_i - \bar{p}_i \le l_i$ &\nm &\cm &\nm &\cm &\nm &\cm &\nm \\
		$\bar{p}_i - p_i \le l_i$ &\nm &\nm &\cm &\cm &\nm &\nm &\cm \\
		$p_i \ge 0$               &\nm &\nm &\nm &\nm &\cm &\cm &\cm \\
		\bottomrule
	\end{tabular}
	\caption{Possible subsets of active constraints in~\eqref{eq:constraints_i}. Check marks indicate active constraints that are included in the basis $B$ for each index $i = 1,\ldots,S$.} \label{tab:constraint_structure}
\end{table}

Implementing a homotopy method in the context of a linear program, such as \eqref{eq:q_linear_program}, is especially convenient since $q(\xi)$ and $\b{p}\opt(\xi)$ are piecewise affine in $\xi$~\cite{Vanderbei2001}. Indeed, the optimal $\b{p}\opt(\xi)$ is affine in $\xi$ for each optimal basis in~\eqref{eq:q_linear_program}, and a breakpoint (or a ``knot'') occurs whenever the currently optimal basis becomes infeasible for a particular $\xi$. This argument also shows that $q (\xi)$ is piecewise affine. Our homotopy method starts with $\xi=0$ and traces an optimal basis in~\eqref{eq:q_linear_program} while increasing $\xi$. The key to its efficiency is the special structure of the relevant bases to problem~\eqref{eq:q_linear_program}, which we describe next.

Each basis $B$ in the linear program~\eqref{eq:q_linear_program} is fully characterized by $2 S$ \emph{linearly independent} (inequality and/or equality) constraints that are \emph{active}, see for example Definition 2.9 of \citeasnoun{Bertsimas1997}. Remember that an active constraint is satisfied with equality, but not every constraint that is satisfied as equality has to be active in a given basis $B$. To analyze the structure of a basis $B$, we note that the components $p_i$ and $l_i$ of any feasible solution $(\bm{p}, \bm{l})$ to~\eqref{eq:q_linear_program} must satisfy the following three inequality constraints:
\begin{equation}\label{eq:constraints_i}
	\begin{aligned}
		p_i - \bar{p}_i &\le l_i ,&
		\bar{p}_i - p_i &\le l_i ,&
		p_i 			&\ge 0~.
	\end{aligned}
\end{equation}
Since the three constraints in~\eqref{eq:constraints_i} contain only two variables $p_i$ and $l_i$, they must be linearly dependent. Thus, for every $i = 1, \ldots, S$, at most two out of the three constraints in~\eqref{eq:constraints_i} can be active. \cref{tab:constraint_structure} enumerates the seven possible subsets of active constraints~\eqref{eq:constraints_i} for any given component $i = 1, \ldots, S$. Here, the letters $\mathcal{N}$, $\mathcal{U}$, $\mathcal{L}$ and $\mathcal{E}$ mnemonize the cases where \underline{n}one of the constraints is active, only the \underline{u}pper bound or the \underline{l}ower bound on $\bar{p}_i$ is active and where both bounds are simultaneously active and hence $p_i$ \underline{e}quals $\bar{p}_i$. Moreover, we have three cases where in addition to the constraints indicated by $\mathcal{N}$, $\mathcal{U}$, $\mathcal{L}$, the nonnegativity constraint $p_i \geq 0$ is active; those cases are distinguished by adding a bar to the aforementioned letters. By construction, the sets in \cref{tab:constraint_structure} are mutually exclusive and jointly exhaustive, that is, they partition the index set $1,\ldots,S$.

In addition to the inequality constraints~\eqref{eq:constraints_i}, a basis $B$ may include one or both of the equality constraints from~\eqref{eq:q_linear_program}. The set $\sQ_B \subseteq \{1,2\}$ indicates which of these equality constraints are included in the basis $B$. Together with the sets from \cref{tab:constraint_structure}, $\sQ_B$ uniquely identifies any basis $B$. The $2S$ linearly independent active constraints involving the $2S$ decision variables uniquely specify a solution $(\bm{p}, \bm{l})$ for a given basis $B$ as
\begin{equation} \label{eq:basis_equalities}
	\begin{aligned}
		p_i - \bar{p}_i &= l_i \quad &\forall i&\in \sU_B\cup\sE_B\cup\sUZ_B \\
		\bar{p}_i - p_i &= l_i \quad &\forall i&\in \sL_B\cup\sE_B\cup\sLZ_B \\
		p_i &= 0 \quad &\forall i&\in \sZ_B\cup\sUZ_B\cup\sLZ_B \\
		\one\tr \b{p} &= 1 \quad &\text{if } 1 &\in \sQ_B \\
		\b{w}\tr \b{l} &= \xi &\text{if } 2 &\in \sQ_B  ~.
	\end{aligned}
\end{equation}
We use $\*p_B(\xi)$ to denote the solution $\*p$ to~\eqref{eq:basis_equalities} and define $q_B(\xi) = \*z\tr \*p_B(\xi)$ for any $\xi$. The vector $\*p_B(\xi)$ may be feasible in~\eqref{eq:q_linear_program} only for some values of $\xi$.

Before we formally characterize the properties of the optimal bases for different values of $\xi$, we illustrate the parametric behavior of $\b{p}\opt(\xi)$, which is an optimizer to \eqref{eq:q_linear_program} that our homotopy algorithm chooses. Note that this optimizer is not necessarily unique. As $\xi$ changes, the values of exactly \emph{two} components of $\b{p}\opt(\xi)$ change. Since the components of $\b{p}\opt (\xi)$ must sum to $1$, one component $p_j$ increases and another component $p_i$ decreases. We say that $p_i$ is a \emph{donor} as it donates some of its probability mass to the \emph{receiver} $p_j$. The examples below illustrate the specific paths traced by $\*p\opt(\xi)$ and illustrate the complications that arise from using non-uniform weights $\*w$.

\begin{figure}
	\centering
	\begin{minipage}{0.49\linewidth}
		\centering
		\includegraphics[width=0.95\linewidth]{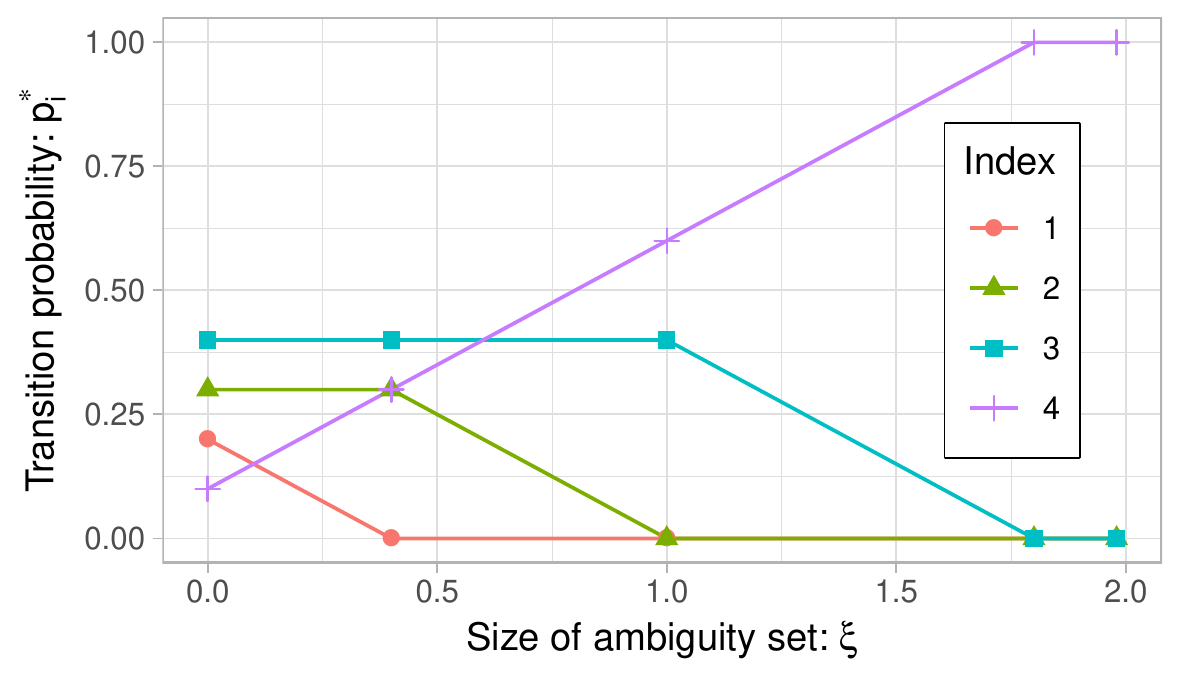}	
	\end{minipage}
	\begin{minipage}{0.49\linewidth}
		\centering
		\includegraphics[width=0.95\linewidth]{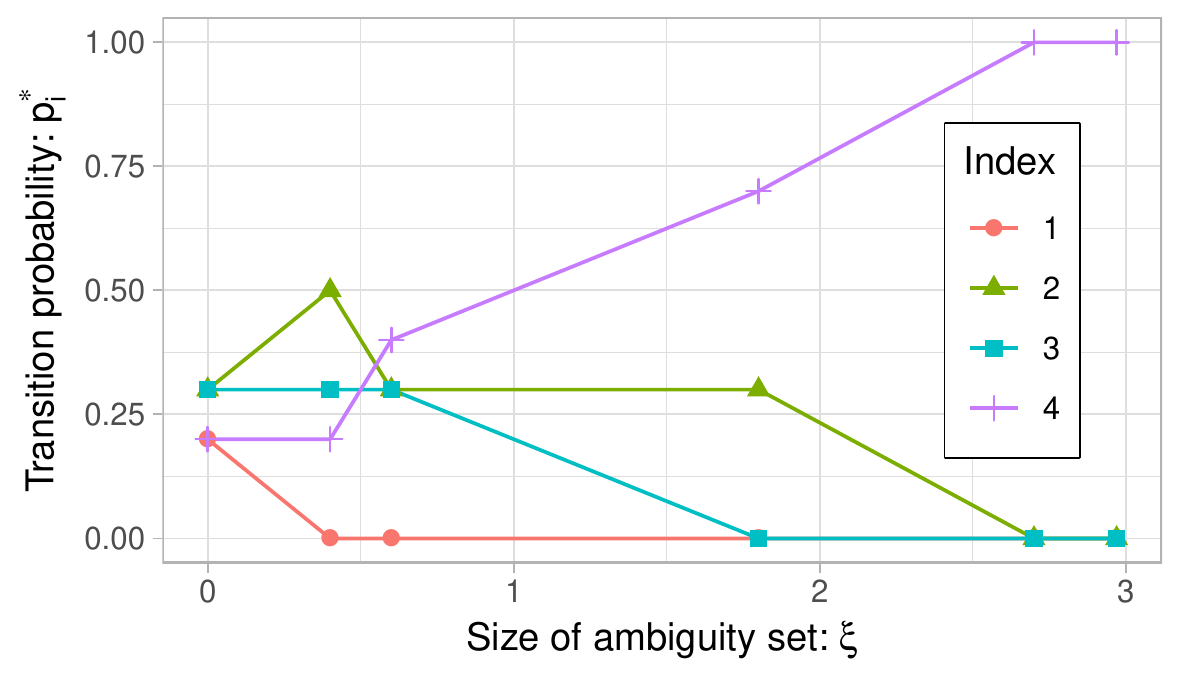}
	\end{minipage}
	\caption{Example evolution of $\b{p}\opt(\xi)$ for a uniform (left) and a non-uniform weight vector $\*w$ (right). Point markers indicate breakpoints where the optimal bases change.} \label{fig:example1_q_homotopy} \label{fig:example2_q_homotopy}
\end{figure}

\begin{exm}[Uniform Weights] \label{exm:unweighted}
	Consider the function $q (\xi)$ in \eqref{eq:q_optimization_l1} for an RMDP with $4$ states, $\b{z} = (4, 3, 2, 1)^\top$, $\bar{\b{p}} = (0.2, 0.3, 0.4, 0.1)^\top$ and $\b{w} = \one$.  \Cref{fig:example1_q_homotopy}~(left) depicts the evolution of $\b{p}\opt (\xi)$ as a function of $\xi$. Component $p_4$ is the receiver for all values of $\xi$, and the donors are the components $p_1$, $p_2$ and $p_3$. We show in \cref{sec:homotopy_complexity} that for uniform weights $\*w$, the component with the smallest value of $\b{z}$ is always the sole receiver.
\end{exm}

\begin{exm}[Non-Uniform Weights] \label{exm:weighted}
	Consider the function $q (\xi)$ in \eqref{eq:q_optimization_l1} for an RMDP with $4$ states, $\b{z} = (2.9, 0.9, 1.5, 0.0)^\top$, $\bar{\b{p}} = (0.2, 0.3, 0.3, 0.2)^\top$ and $w = (1, 1, 2, 2)^\top$. \Cref{fig:example2_q_homotopy}~(right) depicts the evolution of $\b{p}\opt (\xi)$ as a function of $\xi$. The donor-receiver pairs are $(1,2)$, $(2,4)$ $(3,4)$ and again $(2, 4)$. In particular, several components can serve as receivers for different values of $\xi$ when $\*w$ is non-uniform. Also, the same component can serve as a donor more than once.
\end{exm}

In the remainder of this subsection, we show that for any basis $B$ to~\eqref{eq:q_linear_program} that is of interest for our homotopy method, at most two components of $\*p_B(\xi)$ vary with $\xi$. To this end, we bound the sizes of the sets from \cref{tab:constraint_structure}.
\begin{lem} \label{lem:basic_feasible}
	Any basis $B$ to~\eqref{eq:q_linear_program} satisfies $|\sU_B| + |\sL_B| + |\sZ_B| + 2|\sN_B| = |\sQ_B| \le 2$.
\end{lem}

\begin{proof}
	The statement follows from a counting argument. Since the sets listed in \cref{tab:constraint_structure} partition the index set $1, \ldots, S$, their cardinalities must sum to $S$:
	\begin{equation} \label{eq:basis_variables}
		|\sN_B| + |\sU_B| + |\sL_B| + |\sUL_B| + |\sZ_B| + |\sUZ_B| + |\sLZ_B| = S. 
	\end{equation}
	Each index $i = 1,\ldots, S$ contributes between zero and two active constraints to the basis. For example, $i\in\sN_B$ contributes no constraint, whereas $i\in\sUZ_B$ contributes $2$ constraints. The requirement that $B$ contains exactly $2 S$ linearly independent constraints translates to
	\begin{equation}\label{eq:basis_constraints}
		0\cdot|\sN_B| + 1\cdot|\sU_B| + 1\cdot|\sL_B| + 2\cdot|\sUL_B| + 1\cdot|\sZ_B| + 2\cdot|\sUZ_B| + 2\cdot|\sLZ_B| + |\sQ_B| = 2S ~.
	\end{equation}
	Subtracting two times~\eqref{eq:basis_variables} from~\eqref{eq:basis_constraints}, we get
	\[ - 2\cdot|\sN_B| -|\sU_B| - |\sL_B| - |\sZ_B| + |\sQ_B| = 0  ~.\]
	The result then follows by performing elementary algebra.
\end{proof}

We next show that for any basis $B$ feasible in the problem~\eqref{eq:q_linear_program} for a given $\xi$, the elements in $\sU_B$ and $\sL_B$ act as donor-receiver pairs.


\begin{prop} \label{lem:structure}
	Consider some $\xi > 0$ and a basis $B$ to problem \eqref{eq:q_linear_program} that is feasible in a neighborhood of $\xi$. Then the derivatives $\dot{\b{p}} = \frac{d}{d \xi} \b{p}_B(\xi)$ and $\dot{q} = \frac{d}{d \xi} q_B (\xi)$ satisfy:
	\begin{enumerate}[nosep]
		\item[\emph{(C1)}] If $\sU_B = \{ i \}$ and $\sL_B = \{ j \}$, $i \neq j$, then:
		\begin{small}
			\begin{align*} 
				\dot{q} &= \frac{z_i - z_j}{w_i + w_j}, & \dot{p}_i &= \frac{1}{w_i + w_j}, & \dot{p}_j &= -\frac{1}{w_i + w_j}.
			\end{align*}
		\end{small}
		\item[\emph{(C2)}] If $\sU_B = \{i,j\}$, $i \neq j$ and $w_i \neq w_j$, and $\sL_B = \emptyset$, then:
		\begin{small}
			\begin{align*} 
				\dot{q} &= \frac{z_i - z_j}{w_i - w_j}, & \dot{p}_i &= \frac{1}{w_i - w_j}, & \dot{p}_j = -\frac{1}{w_i - w_j}.	
			\end{align*}
		\end{small}
	\end{enumerate}
	The derivatives $\dot{\bm{p}}$ and $\dot{q}$ of all other types of feasible bases to problem~\eqref{eq:q_linear_program} are zero. 
\end{prop}

The derivative $\dot{\bm{p}}$ shows that in a basis of class (C1), $i$ is the receiver and $j$ is the donor. In a basis of class (C2), on the other hand, an inspection of $\dot{\bm{p}}$ reveals that $i$ is the receiver and $j$ is the donor whenever $w_i > w_j$, and the reverse situation occurs when $w_i < w_j$.

\begin{proof}[Proof of~\cref{lem:structure}]
	In this proof, we consider a fixed basis $B$ and thus drop the subscript $B$ to reduce clutter. We also denote by $\*x_{\mathcal{D}}$ the subvector of $\bm{x} \in \Real^S$ formed by the elements $x_i$, $i \in \mathcal{D}$, whose indices are contained in the set $\mathcal{D} \subseteq \states$.
	
	Note that $i \in \sZ \cup \sUZ \cup \sLZ$ implies  $(\bm{p}_B (\xi))_i = 0$ for every $\xi$ and thus $\dot{p}_i = 0$. Likewise, $i \in \sUL$ implies that $(\bm{p}_B (\xi))_i = \bar{p}_i$ for every $\xi$ and thus $\dot{p}_i = 0$ as well. Hence, $\dot{p}_i \neq 0$ is only possible if $i \in \sU \cup \sL \cup \sN$. Since at least two components of $\bm{p}_B (\xi)$ need to change as we vary $\xi$, we can restrict ourselves to bases $B$ that satisfy $| \sU | + | \sL | + |\sN| \geq 2$. Since \cref{lem:basic_feasible} furthermore shows that $| \sU | + | \sL | + 2 |\sN| \leq 2$, we only need to consider three cases in the following: \emph{(C1)} $| \sU | = | \sL | = 1$ and $| \sN | = 0$; \emph{(C2)} $| \sU | = 2$ and $| \sL | = | \sN | = 0$; and \emph{(C3)} $| \sL | = 2$ and $| \sU | = | \sN | = 0$. For each of these cases, we denote by $\bm{p}$ and $\bm{l}$ the unique vectors that satisfy the active constraints~\eqref{eq:basis_equalities} for the basis $B$.

	 \cref{tab:constraint_structure} implies the following useful equality that any $\*p$ must satisfy.
	\begin{equation}\label{eq:element_sum}
		\begin{aligned} 
			1 = \one\tr \*p &= \one\tr \*p_{\sN} + \one\tr \*p_{\sU} + \one\tr \*p_{\sL} + \one\tr \*p_{\sE} + \one\tr \*p_{\sZ} +\one\tr \*p_{\sUZ} + \one\tr \*p_{\sLZ} \\
			&= \one\tr \*p_{\sN} + \one\tr \*p_{\sU} + \one\tr \*p_{\sL} + \one\tr \bar{\*p}_{\sE}
		\end{aligned}
	\end{equation}

	\emph{Case (C1); $\sU = \{ i \}$, $\sL = \{ j \}$, $i \neq j$, and $\sN = \emptyset$:}
	In this case, equation~\eqref{eq:element_sum} implies that $p_i + p_j = 1 - \one\tr\bar{\*p}_\sE$ and thus $\dot{p}_i + \dot{p}_j = 0$. We also have
	\begin{align*}
		\b{w}\tr \b{l} &=  \*w_{\sN}\tr \*l_{\sN} + \*w_{\sU}\tr \*l_{\sU} + \*w_{\sL}\tr \*l_{\sL} + \*w_{\sE}\tr \*l_{\sE} + \*w_{\sZ}\tr \*l_{\sZ} + \*w_{\sUZ}\tr \*l_{\sUZ} + \*w_{\sLZ}\tr \*l_{\sLZ} \\
		&= w_i l_i + w_j l_j + \b{w}_{\sUL}\tr \*l_{\sUL} + \*w_{\sUZ}\tr \*l_\sUZ + \*w_{\sLZ}\tr \*l_{\sLZ} \\
		&= w_i l_i + w_j l_j - \*w_{\sUZ}\tr \bar{\*p}_\sUZ + \*w_{\sLZ}\tr \bar{\*p}_{\sLZ} \\
		&= w_i (p_i - \bar{p}_i) + w_j (\bar{p}_j - p_j) - \*w_{\sUZ}\tr \bar{\*p}_\sUZ + \*w_{\sLZ}\tr \bar{\*p}_{\sLZ} ~,
	\end{align*}
	where the second identity follows from the fact that $\sN = \emptyset$, $\sU = \{ i \}$ and $\sL = \{ j \}$ by assumption, as well as $\sZ = \emptyset$ due to \cref{lem:basic_feasible}. The third identity holds since the active constraints in $\sE$, $\sUZ$ and $\sLZ$ imply that $\*l_{\sE} = \bm{0}$, $\bm{l}_{\sUZ} = - \bar{\b{p}}_\sUZ$ and $\bm{l}_{\sLZ} = \bar{\b{p}}_\sLZ$, respectively. The last identity, finally, is due to the fact that $p_i - \bar{p}_i = l_i$ since $i \in \mathcal{U}$ and $\bar{p}_j - p_j = l_j$ since $j \in \mathcal{L}$. Since any feasible basis $B$ satisfies that $\b{w}\tr \b{l} = \xi$, we thus obtain that
	\[
	\begin{array}{r@{\displaystyle\quad}r@{\displaystyle\; = \;}l@{\qquad}l}
	&  w_i (p_i - \bar{p}_i) + w_j (\bar{p}_j - p_j)
	& \xi + \*w_{\sUZ}\tr \bar{\*p}_\sUZ - \*w_{\sLZ}\tr \bar{\*p}_{\sLZ} \\
	\Longrightarrow
	& w_i \dot p_i - w_j \dot p_j
	& 1
	& \text{taking $\nicefrac{d}{d\xi}$ on both sides} \\
	\Longleftrightarrow
	& w_i \dot p_i + w_j \dot p_i
	& 1
	& \text{from } \dot{p}_i + \dot{p}_j = 0 \\
	\Longleftrightarrow
	& \dot p_i
	& \frac{1}{w_i + w_j} .
	\end{array}
	\]
	The expressions for $\dot p_j$ and $\dot q$ follow from $\dot{p}_i + \dot{p}_j = 0$ and elementary algebra, respectively.
		
	\emph{Case (C2); $\sU = \{ i, j \}$, $i \neq j$, and $\sL = \sN = \emptyset$:}
	Similar steps as in case (C1) show that
	\[ w_i (p_i - \bar{p}_i) + w_j (p_j - \bar{p}_j) = \xi + \b{w}_{\sUZ}\tr \bar{\b{p}}_\sUZ - \b{w}_{\sLZ}\tr \bar{\b{p}}_\sLZ~, \]
	which in turn yields the desired expressions for $\dot{p}_i$, $\dot{p}_j$ and $\dot q$. Note that if $w_i = w_j$ in the equation above, then the left hand side's derivative with respect to $\xi$ is zero, and we obtain a contradiction. This allows us to assume that $w_i \neq w_j$ in case (C2).
	
	\emph{Case (C3); $\sL = \{ i, j \}$, $i \neq j$, and $\sU =\sN = \emptyset$:}
	Note that $\*p_{\sL} \le \bar{\*p}_{\sL}$ since $\bm{l}_{\sL}$ satisfies both $\bm{l}_{\sL} \geq \bm{0}$ and $\bm{l}_{\sL} = \bar{\*p}_{\sL} - \*p_{\sL}$. Since~\eqref{eq:element_sum} implies that $\one\tr \*p = \one\tr \*p_{\sL} + \one\tr \bar{\*p}_{\sE} = 1$, however, we conclude that $\*p_{\sL} = \bar{\*p}_{\sL}$, that is, we must have $\dot{\bm{p}} = \bm{0}$ and $\dot{q} = 0$.
\end{proof}


\subsection{Homotopy Algorithm}\label{sec:homo}

\begin{algorithm}
	\KwIn{LP parameters $\b{z}$, $\b{w}$ and $\bar{\b{p}}$}
	\KwOut{Breakpoints $(\xi_t)_{t = 0,\ldots T+1}$ and values $(q_t)_{t = 0,\ldots T+1}$, defining the function $q$}
	Initialize $\xi_0 \gets 0$, $\b{p}_0 \gets \bar{\b{p}}$ and $q_0 \gets q(\xi_0) = \b{p}_0\tr \b{z}$ \;
	~\\
	\tcp{Derivatives $\dot{q}$ for \emph{bases} of \eqref{eq:q_linear_program} (see \cref{lem:structure})} 
	\For{$i = 1\ldots S$}{
		\For{$j = 1 \ldots S$ satisfying $i \neq j$}{
			\emph{Case C1} ($\mathcal{U}_B = \{ i \}$ and $\mathcal{L}_B = \{ j \}$):
			$\alpha_{i,j} \gets (z_i - z_j) / (w_i + w_j)$ \; 
			\emph{Case C2} ($\mathcal{U}_B = \{ i, j \}$): 				
			$\beta_{i,j} \gets (z_i - z_j) / (w_i - w_j)$ if $w_i \neq w_j$ \;
		}
	}
	~\\
	\tcp{Sort derivatives and map to bases (see \cref{lem:structure})}
	Store $(\alpha_{i,j}, \text{C1})$, $i \neq j$ and $\alpha_{i,j} < 0$, and $(\beta_{i,j}, \text{C2})$, $i \neq j$ and $\beta_{i,j} < 0$, in a list $\mathcal{D}$ \;
	Sort the list $\mathcal{D}$ in ascending order of the first element \; 	
	Construct bases $B_1, \ldots, B_T$ from $\mathcal{D} = (d_1, \dots, d_T)$ as:\\
	$\qquad B_m = \begin{cases}
	(\mathcal{U}_B = \{ i \}, \, \mathcal{L}_B = \{ j \}) & \textbf{if } d_m = (\alpha_{i,j}, \text{C1}) \, , \\		
	(\mathcal{U}_B = \{ i, j \}, \, \mathcal{L}_B = \emptyset) & \textbf{if } d_m = (\beta_{i,j}, \text{C2}) \, ; \\
	\end{cases}$
	~\\[4mm]
	\tcp{Trace optimal $\*p_B(\xi)$ with increasing $\xi$}
	\For{$l = 1\ldots T$}{
		\If{$B_l$ infeasible for $\xi_{l-1}$}{
			Set $\xi_l \gets \xi_{l-1}$, $\*p_l \gets \*p_{l-1}$ and $q_l\gets q_{l-1}$ \;
			\textbf{continue}\;
		}
		Compute $\dot{\*p},\dot{q}$ according to the cases (C1) and (C2) from \cref{lem:structure} \;
		Compute maximum $\Delta \xi$ for which $B_l$ remains feasible:
		$\Delta\xi \gets$\\
		$\qquad  \begin{cases}
		\max \{\Delta \xi \ge 0 \ss (\*p_{l-1})_j + \Delta \xi \cdot \dot{p}_j \ge 0 \}
		&\quad \text{if } d_l = (\alpha_{i,j}, \text{C1}) \, , \\
		\max \{\Delta \xi \ge 0 \ss (\*p_{l-1})_j + \Delta \xi \cdot \dot{p}_j \ge \bar{p}_j \}
		&\quad \text{if } d_l = (\beta_{i,j}, \text{C2}) \text{ and } w_i > w_j \, , \\
		\max \{\Delta \xi \ge 0 \ss (\*p_{l-1})_i + \Delta \xi \cdot \dot{p}_i \ge \bar{p}_i \}
		&\quad \text{if } d_l = (\beta_{i,j}, \text{C2}) \text{ and } w_i < w_j \, ; \\
		\end{cases}$
		
		Set $\xi_l \gets \xi_{l-1} + \Delta \xi$, $\*p_l \gets \*p_{l-1} + \Delta\xi \cdot \dot{\*p}$ and $q_l\gets q_{l-1} + \Delta \xi \cdot \dot{q}$ \;
	}
	Set $\xi_{T+1} \gets \infty$ and $q_{T+1} \gets q_{T}$\;
	\Return{Breakpoints $(\xi_t)_{t = 0,\ldots T+1}$ and values $(q_t)_{t = 0,\ldots T+1}$.}
	\caption{Homotopy method to compute $q(\xi)$.} \label{alg:homotopy}
\end{algorithm}

We are now ready to describe our homotopy method, which is presented in \cref{alg:homotopy}. The algorithm starts at $\xi_0 = 0$ with the optimal solution $\bm{p}_0 = \bar{\bm{p}}$ achieving the objective value $q_0 = \bm{p}_0^\top \bm{z}$. The algorithm subsequently traces each optimal basis as $\xi$ increases, until the basis becomes infeasible and is replaced with the next basis. Since the function $q (\xi)$ is convex, it is sufficient to consider bases that have a derivative $\dot{q}$ that is no smaller than ones traced previously. Note that a basis of class (C1) satisfies $\mathcal{U}_B = \{ i \}$ and $\mathcal{L}_B = \{ j \}$ for some receiver $i \in S$ and some donor $j \in S$, $j \neq i$, and this basis is feasible at $\bm{p} = \bm{p}^\star (\xi)$, $\xi \geq 0$, only if $p_i \in [\bar{p}_i, 1]$ and $p_j \in [0, \bar{p}_j]$ (see~\cref{lem:structure}). Likewise, a basis of class (C2) satisfies $\mathcal{U}_B = \{ i, j \}$, $i \neq j$, and $\mathcal{L}_B = \emptyset$, and it is feasible at $\bm{p} = \bm{p}^\star (\xi)$, $\xi \geq 0$, only if $p_i \in [\bar{p}_i, 1]$ and $p_j \in [\bar{p}_j, 1]$. In a basis of class (C2), $i$ is the receiver and $j$ is the donor whenever $w_i > w_j$, and the reverse situation occurs when $w_i < w_j$. To simplify the exposition, we assume that all bases in \cref{alg:homotopy} have pairwise different slopes $\dot q$, which can always be achieved by applying a sufficiently small perturbation to $\bm{w}$ and/or $\bm{z}$. Our implementation accounts for floating-point errors by using a queue to store and examine the feasibility of all bases that are withing some small $\epsilon$ of the last $\dot q$.

\Cref{alg:homotopy} generates the entire solution path of $q(\xi)$. If the goal is to compute the function $q$ for a particular value of $\xi$, then we can terminate the algorithm once the for loop over $l$ has reached this value. In contrast, our bisection method for s-rectangular ambiguity sets (described in the next section) requires the entire solution path to compute robust Bellman policy updates. We also note that \cref{alg:homotopy} records all vectors $\* p_1, \ldots \* p_T$. This is done for ease of exposition; for practical implementations, it is sufficient to only store the current iterate $\* p_l$ and update the two components that change in the for loop over $l$.

The following theorem proves the correctness of our homotopy algorithm. It shows that the function $q$ is a piecewise affine function defined by the output of \cref{alg:homotopy}.

\begin{thm}\label{thm:homotopy_works}
	Let $(\xi_t)_{t = 0, \ldots, T + 1}$ and $(q_t)_{t = 0, \ldots, T + 1}$ be the output of \cref{alg:homotopy}. Then, $q (\xi)$ is a piecewise affine function with breakpoints $\xi_l$ that satisfies $q (\xi_t) = q_t$ for $t = 0, \ldots, T+1$.
\end{thm}

We prove the statement by contradiction. Since each point $q_l$ returned by \cref{alg:homotopy} corresponds to the objective value of a feasible solution to problem~\eqref{eq:q_linear_program} at $\xi = \xi_l$, the output generated by \cref{alg:homotopy} provides an upper bound on $q (\xi)$. Assume to the contrary that the output does not coincide point-wise with the function $q (\xi)$. In that case, there must be a value of $\xi$ at which the homotopy method disregards a feasible basis that has a strictly smaller derivative than the one selected. This, however, contradicts the way in which bases are selected by the algorithm.

\begin{proof}[Proof of \cref{thm:homotopy_works}]
	For $\xi \leq \xi_T$, the piecewise affine function computed by \cref{alg:homotopy} is
	\[
		g(\xi) = \min_{\*\alpha \in \Delta^{T+1}} \, \left\{ \sum_{t=0}^T \alpha_t \, q_t \ss \sum_{t=0}^T \alpha_t \, \xi_t = \xi \right\}.
	\]
	To prove the statement, we show that $g(\xi) = q(\xi)$ for all $\xi \in [0, \xi_T]$. Note that $g (\xi) \geq q (\xi)$ for all $\xi \in [0, \xi_T]$ by construction since our algorithm only considers feasible bases. Also, from the construction of $g$, we have that $q(\xi_0) = g(\xi_0)$ for the initial point.
	
	To see that $g (\xi) \leq q (\xi)$, we need to show that \cref{alg:homotopy} does not skip any relevant bases. To this end, assume to the contrary that there exists a $\xi' \in (\xi_0, \xi_T]$ such that $q(\xi') < g(\xi')$. Without loss of generality, there exists a value $\xi'$ such that that $q (\xi) = g (\xi)$ for all breakpoints $\xi \leq \xi'$ of $q$; this can always be achieved by choosing a sufficiently small value of $\xi'$ where $q$ and $g$ differ. Let $\xi_l$ be the largest element in $\{ \xi_t \ss t = 0, \ldots, T \}$ such that $\xi_l < \xi'$, that is, we have $\xi_l < \xi' \le \xi_{l+1}$. Such $\xi_l$ exists because $\xi' > \xi_0$ and $q(\xi_0) = g(\xi_0)$. Let $B_l$ be the basis chosen by \cref{alg:homotopy} for the line segment connecting $\xi_l$ and $\xi_{l+1}$. We then observe that
	\[  \dot q(\xi') \;\; = \;\; \frac{q(\xi') - q_l}{\xi' - \xi_l} \;\; < \;\; \frac{g(\xi') - q_l}{\xi' - \xi_l} \;\; = \;\; \frac{q_{l+1} - q_l}{\xi_{l+1} - \xi_l} \;\; = \;\; \dot g(\xi') ~, \]
	where the first identity follows from our choice of $\xi'$, the inequality directly follows from $q(\xi') < g(\xi')$, and the last two identities hold since $B_l$ is selected by \cref{alg:homotopy} for the line segment connecting $\xi_l$ and $\xi_{l+1}$. However, by \cref{lem:basic_feasible,lem:structure}, $B_l$ is the basis with the minimal slope between $\xi_l$ and $\xi_{l+1}$, and it thus satisfies
	\[\frac{q_{l+1} - q_l}{\xi_{l+1} - \xi_l} \;\; \leq \;\; \dot q(\xi) ~,\]
	which contradicts the strict inequality above. The correctness of the last value $\xi_{T+1} = \infty$, finally, follows since $q$ is constant for large $\xi$ as the constraint $\b{w}\tr \*l = \xi$ is inactive.
\end{proof}

\subsection{Complexity Analysis} \label{sec:homotopy_complexity}

A naive implementation of \cref{alg:homotopy} has a computational complexity of $\bigO (S^2 \log S)$ because it sorts all pairs of indexes $(i, j) \in \mathcal{S} \times \mathcal{S}$ according to their derivatives $\dot q$. Although this already constitutes a significant improvement over the theoretical $\bigO (S^{4.5})$ complexity of solving~\eqref{eq:q_linear_program} using a generic LP solver, we observed numerically that the naive implementation performs on par with state-of-the-art LP solvers. In this section, we describe a simple structural property of the parametric problem~\eqref{eq:q_linear_program} that allows us to dramatically speed up \cref{alg:homotopy}.

Our improvement is based on the observation that a component $i \in \mathcal{S}$ cannot be a receiver in an optimal basis if there exists another component $j$ that has both a smaller objective coefficient $z_j$ and weight $w_j$. We call such components $i$ \emph{dominated}, and any dominated receivers can be eliminated from further consideration without affecting the correctness of \cref{alg:homotopy}.

\begin{prop} \label{lem:receivers}
	Consider a component $i \in \states$ such that there is another component $j \in \states$ satisfying $(z_j, w_j) \leq (z_i, w_i)$ as well as $(z_j, w_j) \neq (z_i, w_i)$. Then for any basis $B$ in which $i$ acts as receiver, \cref{alg:homotopy} selects the stepsize $\Delta \xi = 0$.
\end{prop}

\begin{proof}
	Assume to the contrary that in iteration $l$, the basis $B_l$ contains $i$ as receiver and \cref{alg:homotopy} selects a stepsize $\Delta \xi > 0$. Consider $(\xi_{l-1}, \bm{p}_{l-1}, q_{l-1})$, the parameters at the beginning of iteration $l$, as well as $(\xi_l, \bm{p}_l, q_l)$, the parameters at the end of iteration $l$. To simplify the exposition, we denote in this proof by $\one_i$, $i = 1, \ldots, S$, the $i$-th unit basis vector in $\Real^S$.
	
	Let $k \in \mathcal{S}$ be the donor in iteration $l$. Note that $k \neq j$ as otherwise $\dot{q} \geq 0$, which would contradict the construction of the list $\mathcal{D}$. Define $\delta$ via $\bm{p}_l = \bm{p}_{l-1} + \delta [\one_i - \one_k]$, and note that $\delta > 0$ since $\Delta \xi > 0$. We claim that the alternative parameter setting $(\xi_l^\prime, \bm{p}_l^\prime, q_l^\prime)$ with $\bm{p}_l^\prime = \bm{p}_{l-1} + \delta [\one_j - \one_k]$, $\xi_l^\prime = \left \lVert \bm{p}_l^\prime - \bar{\bm{p}} \right \rVert_{1, \bm{w}}$ and $q_l^\prime = \bm{z}^\top \bm{p}_l^\prime$ satisfies $(\xi_l^\prime, q_l^\prime) \leq (\xi_l, q_l)$ and $(\xi_l^\prime, q_l^\prime) \neq (\xi_l, q_l)$. Since this would correspond to a line segment with a steeper decrease than the one constructed by \cref{alg:homotopy}, this contradicts the optimality of \cref{alg:homotopy} proved in \cref{thm:homotopy_works}. To see that $(\xi_l^\prime, q_l^\prime) \leq (\xi_l, q_l)$, note that
	\[
		\xi_l^\prime
		\;\; = \;\;
		\left \lVert \bm{p}_l^\prime - \bar{\bm{p}} \right \rVert_{1, \bm{w}}
		\;\; \leq \;\;
		\left \lVert \bm{p}_l - \bar{\bm{p}} \right \rVert_{1, \bm{w}}
		\;\; = \;\;
		\xi_l
	\]
	since $w_j \leq w_i$ and $p_i \geq \bar{p_i}$ (otherwise, $i$ could not be a receiver). Likewise, we have
	\[
		q_l^\prime
		\;\; = \;\;
		\bm{z}^\top \bm{p}_l^\prime
		\;\; \leq \;\;
		\bm{z}^\top \bm{p}_l
		\;\; = \;\;
		q_l
	\]
	since $z_j \leq z_i$. Finally, since $(w_i, z_i) \neq (w_j, z_j)$, at least one of the previous two inequalities must be strict, which implies that $(\xi_l, \bm{p}_l, q_l)$ is not optimal, a contradiction.
\end{proof}

One readily verifies that if there are two potential receivers $i$ and $j$ satisfying $w_i = w_j$ and $z_i = z_j$, either one of the receivers can be removed from further consideration without affecting the correctness of \cref{alg:homotopy}. We thus arrive at \cref{alg:non-dominated}, which constructs a minimal set of receivers to be considered by \cref{alg:homotopy} in time $\bigO(S\log S)$.

\begin{algorithm}
	\KwIn{Objective coefficients $z_i$ and weights $w_i$ for all components $i \in \mathcal{S}$}
	Sort the elements $z_i$ and $w_i$ in non-decreasing order of $z_i$; break ties in non-decreasing order of $w_i$ \; 
	Initialize the set of possible receivers as $\mathcal{R} \gets \{1\}$ \;
	\For{ $i = 2\ldots S$}{
		\If{ $w_i < \min \{ w_k \ss k\in\mathcal{R} \}$}{
			Update $\mathcal{R} \gets \mathcal{R} \cup \{i\}$ \;
		}
	}
	\Return{Possible receivers mapped back to their original positions in $\mathcal{R}$}
	\caption{Identify non-dominated receivers $i \in \mathcal{S}$.} \label{alg:non-dominated}
\end{algorithm}

\cref{lem:receivers} immediately implies that for a uniform $\b{w}$, only $i \in \mathcal{S}$ with a minimal component $z_i$ can serve as a receiver, and our homotopy method can be adapted to run in time $\bigO(S \log S)$. More generally, if there are $C$ different weight values, then we need to consider at most one receiver for each of the $C$ values. The following corollary summarizes this fact.

\begin{cor} \label{cor:few_pieces}
	If $| \{ w_i \ss i \in \mathcal{S} \} | = C$, then \cref{alg:non-dominated,alg:homotopy} can be adapted to run in  time $\bigO(CS \log CS)$ and produce an output of length $T \leq CS$.
\end{cor}

\section{Computing the Bellman Operator: S-Rectangular Sets} \label{sec:decomposition}

We now develop a bisection scheme to compute the s-rectangular robust Bellman optimality operator $\mathfrak{L}$ defined in~\eqref{eq:bellman_s_rectangular}. Our bisection scheme builds on the homotopy method for the sa-rectangular Bellman optimality operator described in the previous section.

The remainder of the section is structured as follows. We first describe the bisection scheme for computing $\mathfrak{L}$ in \cref{sec:bisection_scheme}. Our method does not directly compute the greedy policy required for our PPI from \cref{sec:ppi} but computes the optimal values of some dual variables instead. \cref{sec:optimal_primal} describes how to extract the optimal greedy policy from these dual variables. Since our bisection scheme for computing $\Bell$ cannot be used to compute the s-rectangular robust Bellman policy update $\Bell_{\bm{\pi}}$ for a fixed policy $\bm{\pi} \in \Pi$, we describe a different bisection technique for computing  $\Bell_{\bm{\pi}}$ in \cref{sec:value_function_update}. We use this technique to solve the robust policy evaluation MDP defined in \cref{sec:ppi}.

\subsection{Bisection Scheme for Robust Bellman Optimality Operator} \label{sec:bisection_scheme}

To simplify the notation, we fix a state $s \in \states$ throughout this section and drop the
associated subscripts whenever the context is unambiguous. In particular, we denote the nominal transition probabilities under action $a$ as $\bar{\b{p}}_a \in\Delta^\statecount$, the rewards under action $a$ as $\b{r}_a\in\Real^S$, the $L_1$-norm weight vector as $\b{w}_a \in\Real^\statecount$, and the budget of ambiguity as $\kappa$. We also fix a value function $\bm{v}$ throughout this section. We then aim to solve the optimization problem
\begin{equation} \label{eq:s_rect_optimization}
\max_{\*d\in\Delta^\actioncount} \min_{\b{\xi}\in\RealPlus^\actioncount} \left\{ \sum_{a\in\actions}  d_a \cdot q_{a}(\xi_a ) \ss \sum_{a\in\actions} \xi_a \le \kappa \right\},
\end{equation}
where $q_a (\xi)$ is defined in~\eqref{eq:q_optimization_l1}. Note that problem~\eqref{eq:s_rect_optimization} exhibits a very specific structure: It has a single constraint, and the function $q_a$ is piecewise affine with at most $S^2$ pieces. We will use this structure to derive an efficient solution scheme that outperforms the naive solution of~\eqref{eq:s_rect_optimization} via a standard LP solver.

\begin{figure}
	\centering
	\includegraphics[width=0.7\linewidth]{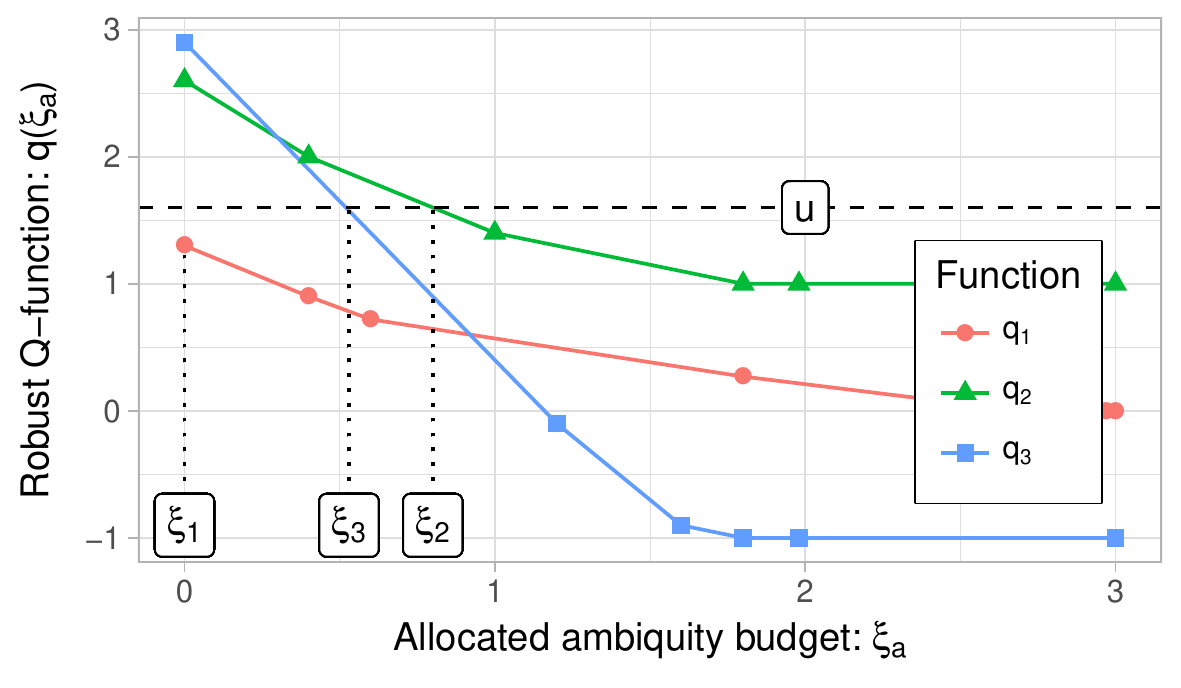}
	\caption{Visualization of the s-rectangular Bellman update with the response functions $q_1, q_2, q_3$ for $3$ actions.} \label{fig:example_decomposition} 
\end{figure}

Our bisection scheme employs the following reformulation of~\eqref{eq:s_rect_optimization}:
\begin{equation} \label{eq:simplified_opt}
	\min_{u \in\Real} \, \left\{ u \ss  \sum_{a\in\actions} q_a^{-1}(u) \le \kappa \right\},
\end{equation}
where the inverse functions $q_a^{-1}$ are defined as
\begin{equation} \label{eq:qa_inverse}
	q_a^{-1}(u) = \min_{\*p \in \Delta^\statecount} \left\{ \| \*p - \bar{\b{p}}_a \|_{1,\b{w}_a} \ss  \b{p}\tr \*z \le u \right\} \qquad \forall a \in \actions.
\end{equation}

Before we formally show that~\eqref{eq:s_rect_optimization} and~\eqref{eq:simplified_opt} are indeed equivalent, we discuss the intuition that underlies the formulation~\eqref{eq:simplified_opt}. In problem~\eqref{eq:s_rect_optimization}, the adversarial nature chooses the transition probabilities $\b{p}_a$, $a \in \mathcal{A}$, to minimize value of $\sum_{a\in\actions} d_a \cdot (\b{p}_a\tr \*z)$ while adhering to the ambiguity budget via $\sum_{a\in\actions} \xi_a \le \kappa$ for $\xi_a = \| \b{p}_a - \bar{\b{p}}_a \|_{1, \b{w}_a}$. In problem~\eqref{eq:qa_inverse}, $q_a^{-1}(u)$ can be interpreted as the minimum ambiguity budget $\| \b{p} - \bar{\b{p}_a} \|_{1,\b{w}_a}$ assigned to the action $a \in \mathcal{A}$ that allows nature to ensure that taking an action $a$ results in a robust value $\b{p}\tr \bm{z}$ not exceeding $u$. Any value of $u$ that is feasible in~\eqref{eq:simplified_opt} thus implies that within the specified overall ambiguity budget of $\kappa$, nature can ensure that \emph{every} action $a \in \mathcal{A}$ results in a robust value not exceeding $u$. Minimizing $u$ in~\eqref{eq:simplified_opt} thus determines the transition probabilities that lead to the lowest robust value under \emph{any} policy, which is the same as computing the robust Bellman optimality operator~\eqref{eq:s_rect_optimization}.

\begin{exm}
	\cref{fig:example_decomposition} shows an example with $3$ actions and the corresponding $q$-functions $q_1,q_2,q_3$. To achieve the robust value of $u$ depicted in the figure, the \emph{smallest} action-wise budgets $\xi_a$ that guarantee $q(\xi_a) \le u$, $i=1,2,3$, are indicated at $\xi_1$, $\xi_2$ and $\xi_3$, resulting in an overall budget of $\kappa = \xi_1 + \xi_2 + \xi_3$.
\end{exm}

We are now ready to state the main result of this section.

\begin{thm} \label{thm:objective_equality}
	The optimal objective values of \eqref{eq:s_rect_optimization} and \eqref{eq:simplified_opt} coincide.
\end{thm}

\Cref{thm:objective_equality} relies on the following auxiliary result, which we state first.

\begin{lem} \label{lem:q_convex}
	The functions $q_a$ and $q_a^{-1}$ are convex in $\xi$ and $u$, respectively.
\end{lem}

\begin{proof}
	The convexity of $q_a$ is immediate from the LP formulation~\eqref{eq:q_linear_program}.  The convexity of $q_a^{-1}$ can be shown in the same way by linearizing the objective function in~\eqref{eq:qa_inverse}.
\end{proof}

\begin{proof}[Proof of \cref{thm:objective_equality}]
	Since the functions $q_a$, $a \in \actions$, are convex (see~\cref{lem:q_convex}), we can exchange the maximization and minimization operators in~\eqref{eq:s_rect_optimization} to obtain
	\[
		\min_{\b{\xi}\in\RealPlus^\actioncount} \left\{ \max_{\*d\in\Delta^\actioncount} \left( \sum_{a\in\actions}  d_a \cdot q_{a}(\xi_a ) \right) \ss \sum_{a\in\actions} \xi_a \le \kappa \right\}.
	\]
	Since the inner maximization is linear in $\*d$, it is optimized at an extreme point of $\Delta^\actioncount$. This allows us to re-express the optimization problem as
	\[
		\min_{\b{\xi}\in\RealPlus^\actioncount} \left\{ \max_{a \in \actions} \left( q_{a} (\xi_a) \right) \ss \sum_{a\in\actions} \xi_a \le \kappa \right\}.
	\]
	We can linearize the objective function in this problem by introducing the epigraphical variable $u \in \Real$:
	\begin{equation}\label{eq:humpdy_dumpdy}
		\min_{u \in \Real} \min_{\b{\xi}\in\RealPlus^\actioncount} \left\{ u \ss \sum_{a\in\actions} \xi_a \le \kappa, \;\; u \geq \max_{a \in \actions} \left[ q_{a} (\xi_a) \right] \right\}.
	\end{equation}
	It can be readily seen that for a fixed $u$ in the outer minimization, there is an optimal $\bm{\xi}$ in the inner minimization that minimizes each $\xi_a$ individually while satisfying $q_a (\xi_a) \leq u$ for all $a \in \actions$. Define $g_a$ as the $a$-th component of this optimal $\bm{\xi}$:
	\begin{equation} \label{eq:g_definition}
		g_a(u) = \min_{\xi_a \in \RealPlus}  \{\xi_a  \ss q_a(\xi_a) \le u \}. 
	\end{equation}
	We show that $g_a(u) = q_a^{-1}(u)$. To see this, we substitute $q_a$ in~\eqref{eq:g_definition} to get:
	\[ g_a(u) = \min_{\xi_a \in \RealPlus} \min_{\b{p}_a \in\Delta^\statecount} \left\{ \xi_a  \ss  \b{p}_a\tr \b{z}_a \le u,\; \norm{\b{p}_a - \bar{\b{p}}_a}_{1,\*w_a} \le \xi_a \right\}~. \]
	The identity $g_a = q_a^{-1}$ then follows by realizing that the optimal $\xi_a\opt$ in the equation above must satisfy $\xi_a\opt = \norm{\b{p}_a - \bar{\b{p}}_a}_{1,\*w_a}$. Finally, substituting the definition of $g_a$ in \eqref{eq:g_definition} into the problem \eqref{eq:humpdy_dumpdy} shows that the optimization problem~\eqref{eq:s_rect_optimization} is indeed equivalent to \eqref{eq:simplified_opt}. 
\end{proof}

\begin{algorithm}
	\KwIn{Desired precision $\epsilon$, functions $q_a^{-1}$, $a\in\actions$\\
		$\quad$ $u_{\min}$: maximum known $u$ for which \eqref{eq:simplified_opt} is \emph{infeasible}, \\
		$\quad$ $u_{\max}$: minimum known $u$ for which \eqref{eq:simplified_opt} is \emph{feasible} }
	\KwOut{$\hat{u}$ such that $\lvert u\opt - \hat{u} \rvert \le \epsilon$, where $u\opt$ is optimal in \eqref{eq:simplified_opt}}
	\While{$u_{\max} - u_{\min} > 2\,\epsilon$}{
		Split interval $[u_{\min}, u_{\max}]$ in half:
		$u \gets (u_{\min} + u_{\max})/ 2$\;
		Calculate the budget required to achieve the mid point $u$: $s \gets \sum_{a\in\actions} q_a^{-1}(u)$ \;
		\eIf{$s \le \kappa$}{
			$u$ is \emph{feasible}: update the feasible upper bound: $u_{\max} \gets u$\;
		}{
			$u$ is \emph{infeasible}: update the infeasible lower bound: $u_{\min} \gets u$\;
		}
	}
	\Return{$(u_{\min} + u_{\max})/2$};
	\caption{Bisection scheme for the robust Bellman optimality operator~\eqref{eq:bellman_s_rectangular}} \label{alg:bisection}
\end{algorithm}

The bisection scheme for solving problem~\eqref{eq:simplified_opt} is outlined in \cref{alg:bisection}. Bisection is a natural and efficient approach for solving the one-dimensional optimization problem. This algorithm is simple and works well in practice, but it can be further improved by leveraging the fact that the functions $q^{-1}_a$, $a \in \actions$, are piecewise affine. In fact, \cref{alg:bisection} only solves problem~\eqref{eq:simplified_opt} to $\epsilon$-optimality, and it requires the choice of a suitable precision $\epsilon$. 

We outline how to adapt \cref{alg:bisection} to determine the optimal solution to problem~\eqref{eq:simplified_opt} in quasi-linear time independent of the precision $\epsilon$; please see \cref{sec:bisection_linear} for details. Recall that \cref{alg:homotopy} computes the breakpoints $(\xi^a_t)_{t = 0, \ldots, T_a+1}$, and objective values $(q^a_t)_{t = 0, \ldots, T_a+1}$, $T_a \leq S^2$, of each function $q_a$, $a \in \actions$. Then each inverse function $q^{-1}_a$ is also piecewise affine with breakpoints $(q^a_t)_{t = 0, \ldots, T_a+1}$, and corresponding function values $\xi_t^a = q^{-1}_a (q_t^a)$. (Care needs to be taken to define $q^{-1}_a (u) = \infty$ for $u < q_{T_a+1}^a$.) We now combine all breakpoints $q^a_t$, $a \in \actions$, to a single list $\mathcal{K}$ in ascending order. We then execute a variant of \cref{alg:bisection} in which both $u_{\min}$ and $u_{\max}$ are always set to some breakpoints from $\mathcal{K}$. Instead of choosing the midpoint $u \gets (u_{\min} + u_{\max})/2$ in each iteration of the bisection, we choose the \emph{median} breakpoint between $u_{\min}$ and $u_{\max}$. We stop once $u_{\min}$ and $u_{\max}$ are consecutive breakpoints in $\mathcal{K}$, in which case the optimal solution of~\eqref{eq:simplified_opt} can be computed by basic algebra. 

The details of \cref{alg:bisection} are described in \cref{sec:bisection_linear} which implies the following complexity statement. 
\begin{thm}\label{thm:p-equals-np}
The combined computational complexity of \cref{alg:homotopy,alg:bisection_linear} is $\bigO (S^2 A \log S A + A\log S \log SA)$.
\end{thm}
Because each execution of \cref{alg:bisection_linear} requires that \cref{alg:homotopy} is executed to produce its inputs, \cref{thm:p-equals-np} states the joint complexity of the two algorithms. Using reasoning similar to \cref{cor:few_pieces}, the bound in \cref{thm:p-equals-np} can be tightened as follows.
\begin{cor}
	If $| \{ w_i \ss i \in \mathcal{S} \} | = C$, then \cref{alg:homotopy,alg:bisection_linear} can be adapted to run jointly in time $\bigO(CSA \log CSA + A \log CS \log CSA)$.
\end{cor}

We emphasize that general (interior-point) algorithms for the linear programming formulation of the robust Bellman optimality operator has the theoretical worst-case complexity of $\bigO(S^{4.5} A^{4.5})$; see \cref{sec:s_rect_linear_formulation}.

\subsection{Recovering the Greedy Policy} \label{sec:optimal_primal}

Since \cref{alg:bisection} only computes the value of the robust Bellman optimality operator $\Bell$ and not an optimal greedy policy $\*d^\star$ achieving this value, it cannot be used in PPI or related robust policy iteration methods \cite{Iyengar2005,Kaufman2013} as is. This section describes how to compute an optimal solution $\*d\opt$ to problem~\eqref{eq:s_rect_optimization} from the output of \cref{alg:bisection}. We again fix a state $s \in \states$ and drop the associated subscripts whenever the context is unambiguous. We also fix a value function $\bm{v}$ throughout this section. Finally, we assume that $\kappa > 0$; the limiting case $\kappa = 0$ is trivial since the robust Bellman optimality operator then reduces to the nominal Bellman optimality operator.

Recall that \cref{alg:bisection} computes the optimal solution $u\opt\in\Real$ to problem~\eqref{eq:simplified_opt}, which thanks to \cref{thm:objective_equality} equals the optimal value of problem~\eqref{eq:s_rect_optimization}. We therefore have
\begin{align}
    \nonumber
    u\opt &= \max_{\*d\in\Delta^\actioncount} \min_{\b{\xi}\in\RealPlus^\actioncount} \left\{ \sum_{a\in\actions}  d_a \cdot q_{a}(\xi_a ) \ss \sum_{a\in\actions} \xi_a \le \kappa \right\} \\
    \label{eq:dual_u_optimal}
    &= \min_{\b{\xi}\in\RealPlus^\actioncount}  \left\{\max_{\*d\in\Delta^\actioncount}  \sum_{a\in\actions}  d_a \cdot q_{a}(\xi_a ) \ss \sum_{a\in\actions} \xi_a \le \kappa \right\}~,
\end{align}
where the second equality follows from the classical Minimax theorem. To compute an optimal $\*d\opt$ from $u^\star$, we first use the definition~\eqref{eq:qa_inverse} of $q_a^{-1}$ to compute $\*\xi\opt$ defined as
\begin{equation} \label{eq:xi_opt_construction}
    \xi_a\opt = q_a^{-1}(u\opt) \qquad \forall a\in\actions~. 
\end{equation}
Intuitively, the components $\xi_a^\star$ of this vector represent the action-wise uncertainty budgets required to ensure that no greedy policy achieves a robust value that exceeds $u^\star$. The set $\mathcal{C}(\*\xi\opt) = \{ a\in\actions \ss q_a(\xi_a\opt) = u\opt\}$ of all actions achieving the optimal robust value plays an important role in the construction of an optimal greedy policy $\*d\opt$. To this end, the following result collects important properties of $\*\xi\opt$ and $\mathcal{C}(\*\xi\opt)$.

\begin{lem}\label{lem:xistar_optimal}
    The vector $\*\xi\opt$ defined in~\eqref{eq:xi_opt_construction} is optimal in~\eqref{eq:dual_u_optimal}. Moreover, $\mathcal{C} (\*\xi\opt) \neq \emptyset$ and
    \begin{enumerate}
        \item[(i)] $q_a(\xi\opt_a) = u^\star$ for all $a \in \mathcal{C} (\*\xi\opt)$;
        \item[(ii)] $\xi_a^\star = 0$ and $q_a(\xi\opt_a) = \bar{\bm{p}}_a^\top \bm{z} \leq u^\star$ for all $a\in\actions\setminus\mathcal{C}(\*\xi\opt)$.
    \end{enumerate}
\end{lem}

\begin{proof}
We first show that $\mathcal{C} (\*\xi\opt) \neq \emptyset$. To this end, we note that for all $a \in \actions$, we have
\[ q_a(\xi_a\opt) = q_a(q_a^{-1}(u\opt)) = \min_{\*p_1 \in\Delta^\statecount} \Bigl\{  \b{p}_1\tr \b{z} \ss \norm{\b{p}_1 - \bar{\b{p}}_a}_{1,\b{w}_a} \le \min_{p_2\in\Delta^S} \left\{ \norm{\b{p}_2 - \bar{\b{p}}_a}_{1,\b{w}_a} \ss \*p_2\tr\*z \le u\opt \right\} \Bigr\} \]
by the definitions of $q_a$ and $q_a^{-1}$ in~\eqref{eq:q_optimization_l1} and~\eqref{eq:qa_inverse}, respectively. Any optimal solution $\*p_2\opt$ to the inner minimization is also feasible in the outer minimization, and therefore $q_a(\xi_a\opt) \le (\*p_2\opt)\tr \*z \le u\opt$. Imagine now that $\mathcal{C} (\*\xi\opt) = \emptyset$. This implies, by the previous argument, that $q_a(\xi\opt_a) < u\opt$ for all $a\in\actions$. In that case, $u^\star$ would not be optimal in~\eqref{eq:simplified_opt} which is a contradiction and therefore $\mathcal{C} (\*\xi\opt) \neq \emptyset$.

We next argue that $\*\xi\opt$ is optimal in~\eqref{eq:dual_u_optimal}. To see that $\*\xi\opt$ is feasible in~\eqref{eq:dual_u_optimal}, we fix any optimal solution $\bar{\*\xi}\in\Real^A$ in~\eqref{eq:dual_u_optimal}. By construction, this solution satisfies $q_a (\bar\xi_a )\le u\opt$ for all $a\in\actions$, and the definition of $q_a$ in~\eqref{eq:q_optimization_l1} implies that there are $\*p_a\in\Delta^S$, $a \in \actions$, such that
$\*p_a\tr \*z \le u\opt$ and $\|\*p_a - \bar{\*p}_a\|_{1,\*w_a} \le \bar{\xi}_a$. The definition of $q_a^{-1}$ in~\eqref{eq:qa_inverse} implies that each $\*p_a$ is feasible in $q_a^{-1}(u\opt)$. Thus, each $\xi_a\opt$ is bounded from above by $\bar{\xi}_a$, and we observe that
\[ \sum_{a\in\actions} \xi\opt_a \;\; \le \;\; \sum_{a\in\actions} \bar{\xi}_a \;\; \le \;\; \kappa ~.\]
Since the definition of $q_a^{-1}$ also implies that $\xi\opt_a = q_a^{-1}(u\opt) \ge 0$, $\*\xi\opt$ is indeed feasible in~\eqref{eq:dual_u_optimal}. The optimality of $\*\xi\opt$ in~\eqref{eq:dual_u_optimal} then follows from the fact that $q_a(\xi\opt_a) \le u\opt$ for all $a \in \actions$.

The statement that $q_a(\xi\opt_a) = u^\star$ for all $a \in \mathcal{C} (\*\xi\opt)$ follows immediately from the definition of $\mathcal{C} (\*\xi\opt)$. To see that $\xi\opt_a = 0$ for $a\in\actions\setminus\mathcal{C}(\*\xi\opt)$, assume to the contrary that $\xi\opt_a > 0$ for some $a\in\actions\setminus\mathcal{C}(\*\xi\opt)$. Since $q_a (\xi\opt_a) < u^\star$, there is $\*p_a\opt \in \Delta^S$ optimal in~\eqref{eq:qa_inverse} satisfying $(\*p_a\opt) \tr \*z < u\opt$ and $\|\*p_a\opt - \bar{\*p}_a\|_{1,\*w_a} \le \xi^\star_a$. At the same time, since $\xi_a^\star > 0$, we have $\|\*p_a\opt - \bar{\*p}_a\|_{1,\*w_a} > 0$ as well. This implies, however, that there is $\epsilon > 0$ such that $\*p_a\opt + \epsilon \cdot(\bar{\*p}_a - \*p_a\opt)$ is feasible in~\eqref{eq:qa_inverse} and achieves a lower objective value than $\*p\opt_a$, which contradicts the optimality of $\*p\opt_a$  in~\eqref{eq:qa_inverse}. We thus conclude that $\xi\opt_a = 0$ for $a\in\actions\setminus\mathcal{C}(\*\xi\opt)$. This immediately implies that $q_a(\xi\opt_a) = \bar{\bm{p}}_a^\top \bm{z}$ for all $a\in\actions\setminus\mathcal{C}(\*\xi\opt)$ as well. The fact that $q_a(\xi\opt_a) \leq u^\star$ for all $a\in\actions\setminus\mathcal{C}(\*\xi\opt)$, finally, has already been shown in the first paragraph of this proof. 
\end{proof}


The construction of $\*d\opt\in \Delta^A$ relies on the slopes of $q_a$, which are piecewise constant but discontinuous at the breakpoints of $q_a$. However, the functions $q_a$ are convex by \cref{lem:q_convex}, and therefore their subdifferentials $\partial q_a(\xi_a)$ exist for all $\xi_a\ge 0$. Using these subdifferentials, we construct optimal action probabilities $\*d\opt\in \Delta^A$ from $\*\xi\opt$ as follows.
\begin{enumerate}[nosep]
    \item[\emph{(i)}] If $0 \in \partial q_{\bar{a}}(\xi_{\bar{a}}\opt)$ for some $\bar{a}\in\mathcal{C}(\*\xi\opt)$, define $\*d\opt$ as
    \begin{subequations}
    \begin{equation} \label{eq:ea_case1}
        d_a\opt = \begin{cases} 1 &\operatorname{if}~ a = \bar{a}  \\
            0 & \operatorname{otherwise}
        \end{cases} 
        \qquad \forall a \in \actions~.
    \end{equation}
    \item[\emph{(ii)}] If $0\notin \partial q_{\bar{a}}(\xi_{\bar{a}}\opt)$ for all $a\in\mathcal{C}(\*\xi\opt)$, define $\*d\opt$ as
    \begin{equation} \label{eq:ea_case2}
        d_a\opt = \frac{e_a}{\sum_{a'\in\actions} e_{a'} } \quad\text{with}\quad 
        e_a = \begin{cases}
            -\frac{1}{f_a} &\operatorname{if}~  a\in\mathcal{C}(\*\xi\opt)  \\
            0  &\operatorname{otherwise}
        \end{cases}
        \qquad \forall a \in \actions~,
    \end{equation}
    \end{subequations}
    where $f_a$ can be any element from $\partial q_a(\xi_a\opt)$, $a \in \actions$.
\end{enumerate}
The choice of $\*d\opt$ may not be unique as there may be multiple $\bar{a}\in\mathcal{C}(\*\xi\opt)$ that satisfy the first condition, and the choice of $f_a \in \partial q_a(\xi_a\opt)$ in the second condition may not be unique either.

\begin{thm} \label{thm:optimal_primal}
    Any vector $\*d\opt$ satisfying~\eqref{eq:ea_case1} or~\eqref{eq:ea_case2} is optimal in problem~\eqref{eq:s_rect_optimization}. Moreover, for $\*\xi\opt$ defined in~\eqref{eq:xi_opt_construction}, $(\*d\opt, \*\xi\opt)$ is a saddle point in~\eqref{eq:s_rect_optimization}.
\end{thm}

\begin{proof}
    One readily verifies that $\bm{d}^\star$ satisfying~\eqref{eq:ea_case1} is contained in $\Delta^A$. To see that $\bm{d}^\star \in \Delta^A$ for $\bm{d}^\star$ satisfying~\eqref{eq:ea_case2}, we note that $\mathcal{C}(\*\xi\opt)$ is non-empty due to \cref{lem:xistar_optimal} and that $f_a < 0$ and thus $e_a > 0$ since $q_a$ is non-increasing. To see that $\*d\opt$ satisfying~\eqref{eq:ea_case1} or~\eqref{eq:ea_case2} is optimal in~\eqref{eq:s_rect_optimization}, we show that it achieves the optimal objective value $u\opt$, that is, that
    \begin{equation}\label{eq:minimization_xi_distar}
        \min_{\b{\xi}\in\RealPlus^\actioncount} \left\{ \sum_{a\in\actions}  d_a\opt \cdot q_{a}(\xi_a ) \ss \sum_{a\in\actions} \xi_a \le \kappa \right\} \ge u\opt~.
    \end{equation}
    Observe that $u\opt$ is indeed achieved for $\*\xi = \*\xi\opt$ since
    \[
        \sum_{a\in\actions}  d_a\opt \cdot q_{a}(\xi_a\opt)
        \;\; = \;\;
        \sum_{a\in\mathcal{C}(\xi\opt)} d_a\opt \cdot q_{a}(\xi_a\opt)
        \;\; = \;\;
        \sum_{a\in\mathcal{C}(\xi\opt)} d_a\opt \cdot u\opt = u\opt~.
    \]
    Here, the first equality holds since $d_a\opt = 0$ for $a\notin\mathcal{C}(\*\xi\opt)$, the second equality follows from the definition of $\mathcal{C}(\*\xi\opt)$, and the third equality follows from $\*d\opt\in\Delta^A$.
    
    To establish the inequality~\eqref{eq:minimization_xi_distar}, we show that $\*\xi\opt$ is optimal in~\eqref{eq:minimization_xi_distar}. This also proves that $(\*d\opt, \*\xi\opt)$ is a saddle point of problem~\eqref{eq:s_rect_optimization}. We denote by $\partial_{\*\xi} (f) [\bm{\xi}^\star]$ the subdifferential of a convex function $f$ with respect to $\*\xi$, evaluated at $\bm{\xi} = \bm{\xi}^\star$. The KKT conditions for non-differentiable convex programs (see, for example, Theorem 28.3 of \citealt{Rockafellar1970}), which are sufficient for the optimality of $\bm{\xi}^\star$ in the minimization on the left-hand side of~\eqref{eq:minimization_xi_distar}, require the existence of a scalar $\lambda\opt\ge 0$ and a vector $\*\alpha\opt\in\RealPlus^A$ such that
    \[
        \begin{array}{c@{\qquad}l}
            \displaystyle \zero \in \partial_{\*\xi} \left(\sum_{a\in\actions} d_a\opt \cdot q_a(\xi_a) - \lambda\opt \left(\kappa - \sum_{a\in\actions} \xi_a\right)  - \sum_{a\in\actions} \alpha_a\opt \cdot \xi_a \right)[\*\xi\opt] \qquad & \text{[Stationarity]} \\
            \displaystyle \lambda\opt \cdot \left(\kappa - \sum_{a\in\actions} \xi_a\opt \right) = 0, \quad \alpha_a\opt \cdot \xi\opt_a = 0 \;\; \forall a \in \actions \qquad & \text{[Compl.~Slackness]}
        \end{array}
    \]
    The stationarity condition simplifies using the chain rule to
    \begin{equation}\label{eq:xi_opt_sim}
    0\in d_a\opt\cdot \partial q_a(\xi_a\opt) + \lambda\opt - \alpha_a\opt \qquad \forall a\in\actions ~.
    \end{equation}
    
    If $\bm{d}^\star$ satisfies~\eqref{eq:ea_case1}, then both~\eqref{eq:xi_opt_sim} and complementary slackness are satisfied for $\lambda^\star = 0$ and $\bm{\alpha}^\star = \bm{0}$. On the other hand, if $\bm{d}^\star$ satisfies~\eqref{eq:ea_case2}, we set
    \[
        \lambda\opt = \frac{1}{\sum_{a\in\mathcal{C}(\*\xi\opt)} e_a }, \qquad 
        \alpha_a\opt = 0 \quad \forall a\in\mathcal{C}(\*\xi\opt), \qquad
        \alpha_a\opt = \lambda\opt \quad \forall a\in\actions\setminus\mathcal{C}(\*\xi\opt)~,
    \]    
    where $e_a$ is defined in~\eqref{eq:ea_case2}. This solution satisfies $\lambda\opt\ge 0$ and $\*\alpha \ge \zero$ because $f_a \le 0$ and therefore $e_a \ge 0$. This solution satisfies~\eqref{eq:xi_opt_sim}, and \cref{lem:xistar_optimal} implies that the second complementary slackness condition is satisfied as well. To see that the first complementary slackness condition is satisfied, we argue that $\sum_{a\in\actions} \xi_a\opt = \kappa$ under the conditions of~\eqref{eq:ea_case2}. Assume to the contrary that $\sum_{a\in\actions} \xi_a\opt < \kappa$. Since $0\notin \partial q_a(\*\xi_a\opt)$ and the sets $\partial q_a(\xi_a\opt)$ are closed for all $a\in\mathcal{C}(\*\xi\opt)$ (see Theorem~23.4 of~\citealt{Rockafellar1970}), we have
    \[
        \exists \bar{\beta}_a > 0
        \quad \text{such that} \quad
        q_a(\xi_a\opt + \beta_a) < q_a(\xi_a) \;\; \forall \beta_a \in (0, \bar{\beta}_a)
    \]
    for all $a\in\mathcal{C}(\*\xi\opt)$. We can thus marginally increase each component $\xi_a^\star$, $a\in\mathcal{C}(\*\xi\opt)$, to obtain a new solution to problem~\eqref{eq:dual_u_optimal} that is feasible and that achieves a strictly lower objective value than $u^\star$. This, however, contradicts the optimality of $u^\star$. We thus conclude that $\sum_{a\in\actions} \xi_a\opt = \kappa$, that is, the first complementary slackness condition is satisfied as well.
\end{proof}

The values $\bm{\xi}^\star$ and $\bm{d}^\star$ can be computed in time $\bigO (A \log S)$ since they rely on the quantities $q_a (\xi_a^\star)$ and $q_a^{-1} (u^\star)$ that have been computed previously by \cref{alg:homotopy} and \cref{alg:bisection}, respectively. The worst-case transition probabilities can also be retrieved from the minimizers of $q_a$ defined in~\eqref{eq:q_optimization_l1} since, as \cref{thm:optimal_primal} implies, $\*\xi\opt$ is optimal in the minimization problem in~\eqref{eq:s_rect_optimization}.

\subsection{Bisection Scheme for Robust Bellman Policy Update} \label{sec:value_function_update}

Recall that the robust policy evaluation MDP $(\mathcal{S}, \bar{\mathcal{A}}, \bm{p}_0, \bar{\bm{p}}, \bar{\bm{r}}, \gamma)$ defined in \cref{sec:ppi} has continuous action sets $\bar\actions(s) = \probs_s$, $s\in\states$, and the transition function $\bar{\*p}$ and the rewards $\bar{\bm{r}}$ defined as
\[ \bar{\*p}_{s, \bm{\alpha}} = \sum_{a\in\actions} \pi_{s,a} \cdot \*\alpha_a \quad \text{and} \quad \bar{r}_{s, \bm{\alpha}} = - \sum_{a\in\actions} \pi_{s,a}\cdot \b{\alpha}_a\tr \b{r}_{s,a}~ . \]
To solve this MDP via value iteration or (modified) policy iteration, we must compute the Bellman optimality operator $\Bell$ defined as
\begin{align*}
    (\Bell \bm{v})_s &= \max_{\bm{\alpha} \in \mathcal{P}_s} \left\{  \bar{r}_{s, \bm{\alpha}} + \gamma \cdot \bar{\*p}_{s, \bm{\alpha}}\tr \bm{v} \right\} \\
    &= \max_{\bm{\alpha} \in (\Delta^S)^A} \left\{  \sum_{a\in\actions} \pi_{s,a} \cdot \*\alpha_a \tr (\gamma \bm{v} - \b{r}_{s,a}) \ss \sum_{a\in\actions} \lVert \b{\alpha}_a - \bar{\b{p}}_{s,a} \rVert_{1,\b{w}_{s,a}} \le \kappa_{s} \right\} \\
    &= - \min_{\bm{\alpha} \in (\Delta^S)^A} \left\{ \sum_{a\in\actions} \pi_{s,a} \cdot \*\alpha_a \tr (\b{r}_{s,a} - \gamma \bm{v}) \ss \sum_{a\in\actions} \lVert \b{\alpha}_a - \bar{\b{p}}_{s,a} \rVert_{1,\b{w}_{s,a}} \le \kappa_{s} \right\}~.
\end{align*}
The continuous action space in this MDP makes it impossible to compute $\Bell \*v$ by simply enumerating the actions. The non-robust Bellman operator could be solved as a linear program, but this suffers from the same computational limitations its application to the robust Bellman operator described earlier.

Using similar ideas as in \cref{sec:bisection_scheme}, we can re-express the minimization problem as
\begin{equation} \label{eq:srect_evaluation}
    \min_{\b{\xi}\in\RealPlus^\actioncount} \left\{ \sum_{a\in\actions}  \pi_{s,a} \cdot q_{s,a}(\xi_a ) \ss \sum_{a\in\actions} \xi_a \le \kappa_s \right\},
\end{equation}
where we use $\bm{z} = \b{r}_{s,a} - \gamma \bm{v}$ in our definition of the functions $q_{s,a}$.

At the first glance, problem~\eqref{eq:srect_evaluation} seems to be a special case of problem~\eqref{eq:s_rect_optimization} from \cref{sec:bisection_scheme}, and one may posit that it can be solved using \cref{alg:bisection}. Unfortunately, this is not the case: In problem~\eqref{eq:srect_evaluation}, the policy $\bm{\pi}$ is fixed and may be \emph{randomized}, whereas \cref{alg:bisection} takes advantage of the fact that $\bm{d}$ can be assumed to be deterministic once the maximization and minimization are swapped in~\eqref{eq:s_rect_optimization}.

Problem~\eqref{eq:srect_evaluation} can still be solved efficiently by taking advantage of the fact that it only contains a single resource constraint on $\b{\xi}$ and that the functions $q_{s,a}$ are piecewise affine and convex. To see this, note that the Lagrangian of \eqref{eq:srect_evaluation} is
\[
    \max_{\lambda \in\RealPlus} \min_{\b{\xi}\in\RealPlus^\actioncount} \left\{ \sum_{a\in\actions} \left( \pi_{s,a} \cdot q_{s,a}(\xi_a) \right) + \lambda \cdot \one\tr  \b{\xi}  - \lambda\,  \kappa_s \right\},
\]
where the use of strong duality is justified since~\eqref{eq:srect_evaluation} can be reformulated as a linear program that is feasible by construction. The minimization can now be decomposed by actions:
\[
    \max_{\lambda \in\RealPlus} \underbrace{\left\{ \sum_{a\in\actions} \min_{\xi_a \in \Real_+} \left\{ \pi_{s,a} \cdot q_{s,a}(\xi_a) + \lambda \xi_a \right\} - \lambda\,  \kappa_s \right\}}_{= u(\lambda)}
\]
The inner minimization problems over $\xi_a$, $a \in \mathcal{A}$, are convex, and they can be solved exactly by bisection since the involved functions $q_{s,a}$ are piecewise affine. Likewise, the maximization over $\lambda$ can be solved exactly by bisection since $u$ is concave and piecewise affine. Note that the optimal value of $\lambda$ is bounded from below by $0$ and from above by the maximum derivative of any $q_{s,a}$, $a \in \mathcal{A}$.

\section{Numerical Evaluation} \label{sec:numerical} 

We now compare the runtimes of PPI (\cref{alg:rpi}) combined with the homotopy method (\cref{alg:homotopy}) and the bisection method (\cref{alg:bisection}) with the runtime of a naive approach that combines the robust value iteration with a computation of the robust Bellman optimality operator $\Bell$ using a general LP solver. We use Gurobi 9.0, a state-of-the-art commercial optimization package. All algorithms were implemented in C++, parallelized using the OpenMP library, and used the Eigen library to perform linear algebra operations. The algorithms were compiled with GCC 9.3 and executed on an AMD Ryzen 9 3900X CPU with 64GB RAM. The source code of the implementation is available at \url{http://github.com/marekpetrik/craam2}.

\subsection{Experimental Setup} \label{sec:domains} 

Our experiments involve two problems from different domains with a fundamentally different structure. The two domains are the \emph{inventory management} problem~\cite{Zipkin200,Porteus2002} and the \emph{cart-pole} problem~\cite{Lagoudakis2003}. The inventory management problem has many actions and dense transition probabilities. The cart-pole problem, on the other hand, has only two actions and sparse transition probabilities. More actions and dense transition probabilities make for much more challenging computation of the Bellman update compared to policy evaluation.

Next, we give a high-level description of both problems as well as our parameter choice. Because the two domains serve simply as benchmark problems and their full description would be lengthy, we only outline their motivation, construction, and properties. To facilitate the reproducibility of the domains, the full source code, which was used to generate them, is available at \url{http://github.com/marekpetrik/PPI_paper}. The repository also contains CSV files with the precise specification of the RMDPs being solved.

In our inventory management problem, a retailer orders, stores and sells a single product over an infinite time horizon. Any orders submitted in time period $t$ will be fulfilled at the beginning of time period $t + 1$, and orders are subject to deterministic fixed and variable costs. Any items held in inventory incur deterministic per-period holding costs, and the inventory capacity is limited. The per-unit sales price is deterministic, but the per-period demand is stochastic. All accrued demand in time period $t$ is satisfied up to the available inventory. Any remaining unsatisfied demand is backlogged at a per-unit backlogging penalty up to a given limit. The states and actions of our MDP represent the inventory levels and the order quantities in any given time period, respectively. The stochastic demands drive the stochastic state transitions. The rewards are the sales revenue minus the purchase costs in each period.

In our experiments, we set the fixed and variable ordering costs to $5.99$ and $1.0$, respectively. The inventory holding and backlogging costs are $0.1$ and $0.15$, respectively. We vary the inventory capacity $I$ to study the impact of the problem's size on the runtimes, while the backlog limit is $I / 3$. We also impose an upper limit of $I / 2$ on each order. The corresponding MDP thus has $I + I/3 = 4/3 \cdot I$ states and $I/2$ actions. Note that due to the inventory capacity limits, not all actions are available at every state. The unit sales price is $1.6$. The demand in each period follows a Normal distribution with a mean of $I/2$ and a standard deviation of $I/5$ and is rounded to the closest integer. We use a discount factor of $0.995$.

In our cart-pole problem, a pole has to be balanced upright on top of a cart that moves along a single dimension. At any point in time, the state of the system is described by four continuous quantities: the cart's position and velocity, as well as the pole's angle and angular velocity. To balance the pole, one can apply a force to the cart from the left or from the right. The resulting MDP thus accommodates a 4-dimensional continuous state space and two actions. Several different implementations of this problem can be found in the literature; in the following, we employ the deterministic implementation from the OpenAI Gym. Again, we use a discount factor of $0.995$. 

Since the state space of our cart-pole problem is continuous, we discretize it to be amenable to our solution methods. The discretization follows a standard procedure in which random samples from the domain are subsampled to represent the discretized state space. The transitions are then estimated from samples that are closest to each state. In other words, the probability of transitioning from a discretized state $s$ to another discretized state $s'$ is proportional to the number of sampled transitions that originate near $s$ and end up near $s'$. The discretized transition probabilities are no longer deterministic, even though the original problem transitions are.

The ambiguity sets are modified slightly in this section to ensure a more realistic evaluation.
Assuming that the robust transition can be positive to any state of the RMDP can lead to overly conservative policies. To obtain less conservative policies, we restrict our ambiguity sets $\mathcal{P}_{s,a}$ and $\mathcal{P}_s$ from \cref{sec:robust_bellman_updates} to probability distributions that are \emph{absolutely continuous} with respect to the nominal distributions $\bar{\b{p}}_{s,a}$. Our sa-rectangular ambiguity sets $\mathcal{P}_{s,a}$ thus become
\[ 
\mathcal{P}_{s,a} = \left\{ \b{p} \in \Delta^\statecount \ss \| \b{p} - \bar{\b{p}}_{s,a} \|_{1,\b{w}_{s,a}} \le \kappa_{s,a}, \;\; p_{s'} \leq \left \lceil \bar{p}_{s,a,s'} \right \rceil \;\; \forall s' \in \mathcal{S} \right\}~,
\]
and we use a similar construction for our s-rectangular ambiguity sets $\mathcal{P}_s$. We set the ambiguity budget to $\kappa_{s,a} = 0.2$ and $\kappa_{s} = 1.0$ in the sa-rectangular and s-rectangular version of our inventory management problem, respectively, and we set $\kappa_{s,a} = \kappa_s = 0.1$ in our cart-pole problem. Anecdotally, the impact of the ambiguity budget on the runtimes is negligible. We report separate results for uniform weights $\*w_{s,a} = \one$ and non-uniform weights $\*w_{s,a}$ that are derived from the value function $\*v$. In the latter case, we follow the suggestions of \citeasnoun{Russel2019} and choose weights $(\*w_{s,a})_{s'}$ that are proportional to $\lvert v_{s'} - \one\tr \*v / S \rvert$. All weights $\*w_{s,a}$ are normalized so that their values are contained in $[0, 1]$. Note that the simultaneous scaling of $\*w_{s,a}$ and $\kappa_{s,a}$ does not affect the solution.

Recall that the policy evaluation step in PPI can be accomplished by any MDP solution method. In our inventory management problem, whose instances have up to $1,000$ states, we use policy iteration and solve the arising systems of linear equations via the LU decomposition of the Eigen library~\cite{Puterman2005}. This approach does not scale well to MDPs with $S \gg 1,000$ states as the policy iteration manipulates matrices of dimension $S\times S$. Therefore, in our cart-pole problem, whose instances have $1,000$ or more states, we use modified policy iteration~\cite{Puterman2005} instead. We compare the performance of our algorithms to the robust value iteration as well as the robust modified policy iteration~(RMPI) of \citeasnoun{Kaufman2013}. Recall that in contrast to PPI, RMPI evaluates robust policies through a fixed number of value iteration steps. Since the impact of the number of value iteration steps on the overall performance of RMPI is not well understood, we fix this number to $1,000$ throughout our experiments. Finally, we set $\epsilon_{k+1} = \min\{ \gamma^2 \epsilon_k, 0.5/(1-\gamma) \cdot \norm{\Bell_{\bm{\pi}_{k}} \b{v}_k - \b{v}_k}_\infty\}$ in \cref{alg:rpi}, which satisfies the convergence condition in \cref{thm:ppi_convergence}. 

\subsection{Results and Discussion} \label{sec:bellman_runtime} 

\begin{table}
 \centering
 \begin{tabular}{|llr|rr|rr|} 
 \toprule
    & & & \multicolumn{2}{c|}{SA-rectangular} & \multicolumn{2}{c|}{S-rectangular} \\
  \midrule
  \multicolumn{1}{|c}{Problem} & \multicolumn{1}{c}{Ambiguity} & \multicolumn{1}{c|}{States} & \multicolumn{1}{c}{LP Solver} & \multicolumn{1}{c|}{\cref{alg:homotopy}} & \multicolumn{1}{c}{LP Solver} & \multicolumn{1}{c|}{\cref{alg:bisection}}\\
  \midrule
  Inventory    & Uniform  &   100 & 13.96 & 0.02 & 24.67 & 0.06 \\  
  Inventory    & Weighted &   100 & 13.85 & 0.75 & 21.36 & 0.86 \\
  Inventory    & Uniform  &   500 & 583.20 & 0.36 & 1,715.94 & 19.65 \\
  Inventory    & Weighted &   500 & 440.35 & 20.69 & 655.00 & 36.24 \\
  Inventory    & Uniform  & 1,000 & $>$ 10,000.00 & 20.00 & $>$ 10,000.00 & 51.97 \\
  Inventory    & Weighted & 1,000 & 4,071.47 & 109.27 & 3,752.21 & 163.32 \\
  \midrule
  Cart-pole    & Uniform  &  1,000 & 9.50 & 0.18 & 19.85 & 1.94 \\  
  Cart-pole    & Weighted &  1,000 & 12.70 & 1.93 & 32.80 & 1.90 \\
  Cart-pole    & Uniform  &  2,000 & 12.81 & 1.90 & 13.33 & 1.88 \\
  Cart-pole    & Weighted &  2,000 & 12.04 & 2.03 & 13.08 & 1.95 \\
  Cart-pole    & Uniform  &  4,000 & 23.39 & 1.91 & 23.29 & 1.76 \\
  Cart-pole    & Weighted &  4,000 & 19.96 & 2.05 & 21.16 & 2.14\\
  \bottomrule
 \end{tabular} 
 \caption{Runtime (in seconds) required by different algorithms to compute 200 steps of the robust Bellman optimality operator.} \label{tab:results_bellman}
\end{table}

\Cref{tab:results_bellman} reports the runtimes required by our homotopy method ({\cref{alg:homotopy}), our bisection method (\cref{alg:bisection}) and Gurobi (LP Solver) to compute 200 steps of the robust Bellman optimality operator $\Bell$ across all states $s \in \mathcal{S}$. We fixed the number of Bellman evaluations in this experiment to clearly separate the speedups achieved by a quicker evaluation of the Bellman operator itself, studied in this experiment, from the speedups obtained by using PPI in place of value iteration, studied in the next experiment. The computations are parallelized over all available threads via OpenMP using Jacobi-style value iteration~\cite{Puterman2005}. By construction, all algorithms identify the same optimal solutions in each application of the Bellman operator. The computations were terminated after $10,000$ seconds.

There are several important observations we can make from the results in \cref{tab:results_bellman}. First of all, that our algorithms outperform Gurobi by an order of magnitude for weighted ambiguity sets and by two orders of magnitude for uniform (unweighted) ambiguity sets, independent of the type of rectangularity. This impressive performance is because the inventory management problem has many actions, which makes computing the Bellman operator particularly challenging. The computation time also reflects that homotopy and bisection methods have quasi-linear time complexities when used with uniform $L_1$ norms. It is remarkable that even with the simple cart-pole problem our algorithms are about 10 to 20 times faster than a state-of-the-art LP solver. Notably, even moderately-sized RMDPs may be practically intractable to general LP solvers.

S-rectangular instances of such problems are particularly challenging for LP solvers as they have to solve a single, monolithic LP across all actions. Perhaps surprisingly, our algorithms also outperform Gurobi in the simple cart-pole problem by an order of magnitude. In fact, the table reveals that even moderately-sized RMDPs may be practically intractable when solved with generic LP solvers.

\begin{table}
 \centering
 \begin{tabular}{|llr|rrr|rr|} 
 \toprule
    & & & \multicolumn{3}{c|}{SA-rectangular} & \multicolumn{2}{c|}{S-rectangular} \\
  \midrule
  \multicolumn{1}{|c}{Problem} & \multicolumn{1}{c}{Ambiguity} & \multicolumn{1}{c|}{States} & \multicolumn{1}{c}{VI} & \multicolumn{1}{c}{RMPI} & \multicolumn{1}{c|}{PPI} & \multicolumn{1}{c}{VI} & \multicolumn{1}{c|}{PPI}\\
  \midrule
  Inventory   & Uniform  & 100  & 0.12 & 0.03 & 0.01 & 3.52 & 0.15 \\ 
  Inventory   & Weighted & 100  & 10.28 & 0.94 & 0.14 & 15.02 & 1.02 \\  
  Inventory   & Uniform  & 500  & 1.39 & 0.06 & 0.14 & 24.69 & 2.71 \\ 
  Inventory   & Weighted & 500  & 140.53 & 5.69 & 2.11 & 276.63 & 16.76 \\ 
  Inventory   & Uniform  & 1,000  & 8.65 & 0.23 & 0.59 & 217.90 & 13.98 \\ 
  Inventory   & Weighted & 1,000  & 393.90 & 14.36 & 6.90 & 519.21 & 163.18 \\
  \midrule  
  Cart-pole   & Uniform  & 1,000  & 0.03 & 0.06 & 0.03 & 0.80 & 0.15 \\ 
  Cart-pole   & Weighted & 1,000  & 0.25 & 0.17 & 0.04 & 0.98 & 0.28 \\  
  Cart-pole   & Uniform  & 10,000  & 0.32 & 0.26 & 0.13 & 8.40 & 1.06 \\ 
  Cart-pole   & Weighted & 10,000  & 1.72 & 1.13 & 0.21 & 13.43 & 3.52 \\  
  Cart-pole   & Uniform  & 20,000  & 0.44 & 0.54 & 0.29 & 16.24 & 2.40 \\ 
  Cart-pole   & Weighted & 20,000  & 6.37 & 3.22 & 0.62 & 28.50 & 9.30 \\  
  \bottomrule
 \end{tabular}
 \caption{Runtime (in seconds) required by different algorithms to compute an approximately optimal robust value function.} \label{tab:results_ppi}
\end{table}

\Cref{tab:results_ppi} reports the runtimes required by the parallelized versions of the robust value iteration (VI), the robust modified policy iteration (RMPI) and our partial policy iteration (PPI) to solve our inventory management and cart-pole problems to approximate optimality. To this end, we choose a precision of $\delta = 40$ (that is, $\norm{\Bell_{\bm{\pi}_{k}} \b{v}_k - \b{v}_k}_\infty \le 0.1$), as defined in \cref{alg:rpi}, for our inventory management problem, as well as a smaller precision of $\delta = 4$ (that is, $\norm{\Bell_{\bm{\pi}_{k}} \b{v}_k - \b{v}_k}_\infty \le 0.01$) for our cart-pole problem, to account for the smaller rewards in this problem. All algorithms use the homotopy (\cref{alg:homotopy}) and the bisection method (\cref{alg:bisection}) to compute the robust Bellman optimality operator. Note that RMPI is only applicable to sa-rectangular ambiguity sets. The computations were terminated after $10,000$ seconds.

There are also several important observations we can make from the results in \cref{tab:results_ppi}. As one would expect, PPI in RMDPs behaves similarly to policy iteration in MDPs. It outperforms value iteration in essentially all benchmarks, being almost up to 100 times faster, but the margin varies significantly. The improvement margin depends on the relative complexity of policy improvements and evaluations. In the sa-rectangular cart-pole problem, for example, the policy improvement step is relatively cheap, and thus the benefit of employing a policy evaluation is small. The situation is reversed in the s-rectangular inventory management problem, in which the policy improvement step is very time-consuming. PPI outperforms the robust value iteration most significantly in the sa-rectangular inventory management problem since the policy evaluation step is much cheaper than the policy improvement step due to the large number of available actions. RMPI's performance, on the other hand, is more varied: while it sometimes outperforms the other methods, it is usually dominated by at least one of the competing algorithms. We attribute this fact to the inefficient value iteration that is employed in the robust policy evaluation step of RMPI.  It is important to emphasize that PPI has the same theoretical convergence rate as the robust value iteration, and thus its performance relative to the robust value iteration and RMPI will depend on the specific problem instance and as well as the employed parameter settings.

In conclusion, our empirical results show that our proposed combination of PPI and the homotopy or bisection method achieves a speedup of up to four orders of magnitude for both sa-rectangular and s-rectangular ambiguity sets when compared with the state-of-the-art solution approach that combines a robust value iteration with a computation of the robust Bellman operator via a commercial LP solver. Since our methods scale more favorably with the size of the problem, their advantage is likely to only increase with larger problems that what we considered here.

\section{Conclusion}\label{sec:conclusion}

We proposed three new algorithms to solve robust MDPs over $L_1$-ball uncertainty sets. Our homotopy algorithm computes the robust Bellman operator over sa-rectangular $L_1$-ball uncertainty sets in quasi-linear time and is thus almost as efficient as computing the nominal, non-robust Bellman operator. Our bisection scheme utilizes the homotopy algorithm to compute the robust Bellman operator over s-rectangular $L_1$-ball uncertainty sets, again in quasi-linear time. Both algorithms can be combined with PPI, which generalizes the highly efficient modified policy iteration scheme to robust MDPs. Our numerical results show significant speedups of up to four orders of magnitude over a leading LP solver for both sa-rectangular and s-rectangular ambiguity sets.

Our research opens up several promising avenues for future research. First, our homotopy method sorts the bases of problem~\eqref{eq:q_linear_program} in quasi-linear time. This step could also be implemented in linear time using a variant of the \emph{quickselect} algorithm, which has led to improvements in a similar context~\cite{Condat2016}. Second, we believe that the techniques presented here can be adapted to other uncertainty sets, such as $L_\infty$- and $L_2$-balls around the nominal transition probabilities or uncertainty sets based on $\phi$-divergences. Both the efficient implementation of the resulting algorithms as well as the empirical comparison of different uncertainty sets on practical problem instances would be of interest. Finally, it is important to study how our methods generalize to robust value function approximation methods~\cite{Tamar2014a}.

\subsection*{Acknowledgments}


We thank Bruno Scherrer for pointing out the connections between policy iteration and algorithms for solving zero-sum games and Stephen Becker for insightful comments. This work was supported by the National Science Foundation under Grants No.~IIS-1717368 and IIS-1815275, by the Engineering and Physical Sciences Research Council under Grant No.~EP/R045518/1, and by the Start-Up Grant scheme of the City University of Hong Kong. Any opinions, findings, and conclusions or recommendations are those of the authors and do not necessarily reflect the views of the funding bodies.

\bibliographystyle{abbrvnatplain}
\bibliography{library}

\appendix

\section{Properties of Robust Bellman Operator} \label{sec:bell_props}

We prove several fundamental properties of the robust Bellman policy update $\Bell_{\bm{\pi}}$ and the robust Bellman optimality operator $\Bell$ over s-rectangular and sa-rectangular ambiguity sets.

\begin{prop} \label{lem:contraction}
    For both s-rectangular and sa-rectangular ambiguity sets, the robust Bellman policy update $\Bell_{\bm{\pi}}$ and the robust Bellman optimality operator $\Bell$ are $\gamma$-contractions under the $L_\infty$-norm, that is
    \[ 
    \norm{\Bell_{\bm{\pi}} \b{x} - \Bell_{\bm{\pi}} \b{y}}_\infty \le \gamma \norm{\b{x} - \b{y}}_\infty
    \qquad \text{and} \qquad
    \norm{\Bell \b{x} - \Bell \b{y}}_\infty \le \gamma \norm{\b{x} - \b{y}}_\infty ~.
    \]
    The equations $\Bell_{\bm{\pi}} \b{v} = \b{v}$ and $\Bell \b{v} = \b{v}$ have the unique solutions $\bm{v}_{\bm{\pi}}$ and $\b{v}\opt$, respectively. 
\end{prop}

\begin{proof}
    See Theorem 3.2 of \citeasnoun{Iyengar2005} for sa-rectangular sets and Theorem 4 of \citeasnoun{Wiesemann2013} for s-rectangular sets.
\end{proof}

\begin{prop} \label{lem:monotone}
    For both s-rectangular and sa-rectangular ambiguity sets, the robust Bellman policy update $\Bell_{\bm{\pi}}$ and the robust Bellman optimality operator $\Bell$ are monotone:
    \[ \Bell_{\bm{\pi}} \b{x} \ge \Bell_{\bm{\pi}} \b{y} \quad \text{ and } \quad \Bell \b{x} \ge \Bell \b{y} \qquad \forall \b{x} \ge \b{y} ~.\]
\end{prop}

\begin{proof}
    We show the statement for s-rectangular ambiguity sets; the proof of sa-rectangular uncertainty sets is analogous. Consider $\bm{\pi} \in \Pi$ as well as $\bm{x}, \bm{y} \in \Real^S$ such that $\bm{x} \geq \bm{y}$ and define
    \[
        F_s(\*p,\*x) = \sum_{a\in\actions} \pi_{s,a} \cdot \*p_a \tr (\*r_{s,a} + \gamma\cdot\*x)~.
    \]
    The monotonicity of the robust Bellman policy update $\Bell_{\bm{\pi}}$ follows from the fact that
    \[
        (\Bell_{\bm{\pi}} \*x)_s = \min_{\*p\in \mathcal{P}_s} \; F_s(\*p,\*x) = F_s(\*p\opt,\*x) \ge F_s(\*p\opt,\*y) \stackrel{\text{(a)}}{\ge} (\Bell_{\bm{\pi}} \*y)_s \qquad \forall s \in \mathcal{S}~,
    \]
    where $\*p\opt \in \mathop{\arg \min}_{\*p\in \mathcal{P}_s} \, F_s(\*p,\*x)$. The inequality (a) holds because $F_s(\*p\opt, \cdot)$ is monotone since $\*p\opt \ge \zero$.
    
    To prove the monotonicity of the robust Bellman optimality operator $\Bell$, consider again some $\bm{x}$ and $\bm{y}$ with $\bm{x} \geq \bm{y}$ and let $\bm{\pi}\opt$ be the greedy policy satisfying $\Bell \bm{y} = \Bell_{\bm{\pi}^\star} \bm{y}$. We then have that
    \[
        (\Bell \*y)_s  = (\Bell_{\bm{\pi}\opt} \*y)_s \le  (\Bell_{\bm{\pi}\opt}\*x)_s \le (\Bell \*x)_s,
    \]
    where the inequalities follow from the (previously shown) monotonicity of $\Bell_{\bm{\pi}^\star}$ and the fact that $(\Bell \*x)_s = (\max_{\*\pi\in\Pi} \Bell_{\*\pi}\*x)_s \ge (\Bell_{\*\pi\opt}\*x)_s $.
\end{proof}

\Cref{lem:contraction,lem:monotone} further imply the following two properties of $\Bell_{\bm{\pi}}$ and $\Bell$.
\begin{cor} \label{cor:optimal_dominates}
    For both s-rectangular and sa-rectangular ambiguity sets, the robust Bellman policy update $\Bell_{\bm{\*\pi}}$ and the robust Bellman optimality operator $\Bell$ satisfy $\bm{v}\opt \geq \*v_{\*\pi}$ for each $\*\pi \in \Pi$.
\end{cor}
\begin{proof}
The corollary follows from the monotonicity (\cref{lem:monotone}) and contraction properties (\cref{lem:contraction}) of $\Bell$ and $\Bell_{\*\pi}$ using standard arguments. See, for example, Proposition 2.1.2 in \citeasnoun{Bertsekas2013}.  
\end{proof}

\begin{cor} \label{lem:opt_value_function_approx}
    For both s-rectangular and sa-rectangular ambiguity sets, the robust Bellman policy update $\Bell_{\bm{\pi}}$ and the robust Bellman optimality operator $\Bell$ satisfy for any $\*v \in\Real^S$ that
    \[
        \norm{\b{v}\opt - \b{v}}_\infty \le \frac{1}{1-\gamma} \norm{\Bell \b{v} - \b{v}}_\infty
        \quad \text{and} \quad
        \norm{\b{v}_{\bm{\pi}} - \b{v}}_\infty \le \frac{1}{1-\gamma} \norm{\Bell_{\bm{\pi}} \b{v} - \b{v}}_\infty~.
    \]
\end{cor}
\begin{proof}
The corollary follows from the monotonicity (\cref{lem:monotone}) and contraction properties (\cref{lem:contraction}) of $\Bell$ and $\Bell_{\*\pi}$ using standard arguments. See, for example, Proposition 2.1.1 in \citeasnoun{Bertsekas2013}.  
\end{proof}

We next show that both $\Bell_{\bm{\pi}}$ and $\Bell$ are invariant when adding a constant to the value function.

\begin{lem} \label{lem:bellman_linear_translation}
    For both s-rectangular and sa-rectangular ambiguity sets, the robust Bellman policy update $\Bell_{\bm{\pi}}$ and the robust Bellman optimality operator $\Bell$ are translation invariant for each $\*\pi\in\Pi$:
    \[ 
    \Bell_{\bm{\pi}} (\*v + \epsilon \cdot \one) = \Bell_{\bm{\pi}} \*v + \gamma \epsilon \cdot \one
    \quad \text{and} \quad
    \Bell (\*v + \epsilon \cdot \one) = \Bell \*v + \gamma \epsilon \cdot \one
    \qquad \forall \*v \in \Real^S, \; \forall \epsilon \in \Real
    \]
\end{lem}

\begin{proof}
    We show the statement for s-rectangular ambiguity sets; the proof of sa-rectangular uncertainty sets is analogous. Fixing $\bm{\pi} \in \Pi$, $\bm{v} \in \Real^S$ and $\epsilon \in \Real$, we have
    \begin{align*} 
        (\Bell_{\bm{\pi}} (\b{v} + \epsilon \one))_s &= \min_{\*p\in\probs_{s}} \sum_{a\in\actions} \pi_{s,a} \cdot \*p_a\tr (\*r_{s,a} + \gamma\cdot [\*v + \epsilon \cdot \one]) \\
        &= \min_{\*p\in\probs_{s}} \sum_{a\in\actions} \pi_{s,a} \cdot (\*p_a\tr (\*r_{s,a} + \gamma\cdot \*v) + \gamma \epsilon ) \\
        &= \gamma \epsilon + \min_{\*p\in\probs_{s}} \sum_{a\in\actions} \pi_{s,a} \cdot \*p_a\tr (\*r_{s,a} + \gamma\cdot \*v) ~,
    \end{align*}
    where the first identity holds by definition of $\Bell_{\bm{\pi}}$, the second is due to the fact that $\*p_a\tr \one = 1$ since $\mathcal{P}_s \subseteq (\Delta^S)^A$, and the third follows from the fact that $\sum_{a \in \actions} \pi_{s,a} = 1$.
    
    To see that $\Bell (\*v + \epsilon \cdot \one) = \Bell \*v + \gamma \epsilon \cdot \one$, we note that
    \[
        \Bell (\*v + \epsilon \cdot \one)
        =
        \Bell_{\bm{\pi}^1} (\*v + \epsilon \cdot \one)
        =
        \Bell_{\bm{\pi}^1} \*v + \gamma \epsilon \cdot \one
        \leq
        \Bell \*v + \gamma \epsilon \cdot \one~,
    \]
    where $\bm{\pi}^1 \in \Pi$ is the greedy policy that satisfies $\Bell_{\bm{\pi}^1} (\*v + \epsilon \cdot \one) = \Bell (\*v + \epsilon \cdot \one)$, as well as
    \[
        \Bell \*v + \gamma \epsilon \cdot \one
        =
        \Bell_{\bm{\pi}^2} \*v + \gamma \epsilon \cdot \one
        =
        \Bell_{\bm{\pi}^2} (\*v + \epsilon \cdot \one)
        \leq
        \Bell (\*v + \epsilon \cdot \one)~,
    \]
    where $\bm{\pi}^2 \in \Pi$ is the greedy policy that satisfies $\Bell_{\bm{\pi}^2} \*v = \Bell \*v$.
\end{proof}

Our last result in this section shows that the difference between applying the robust Bellman policy update $\Bell_{\bm{\pi}}$ to two value functions can be bounded from below by a linear function.

\begin{lem} \label{lem:Bellman_bound_linear}
    For both s-rectangular and sa-rectangular ambiguity sets, there exists a stochastic matrix $\*P$ such that the robust Bellman policy update $\Bell_{\bm{\pi}}$ satisfies
    \[ \Bell_{\bm{\pi}} \*x - \Bell_{\bm{\pi}} \*y \ge \gamma \cdot \*P (\*x - \*y)~, \]
    for each $\*\pi\in\Pi$ and $\*x, \*y \in \Real^S$.
\end{lem}

\begin{proof}
    We show the statement for s-rectangular ambiguity sets; the proof of sa-rectangular uncertainty sets is analogous. We have that
    \begin{align*}
        (\Bell_{\bm{\pi}} \*x - \Bell_{\bm{\pi}} \*y)_s &= \min_{\*p \in \probs_s} \left\{ \sum_{a\in\actions} \pi_{s,a} \cdot \*p_a\tr (\*r_{s,a} + \gamma\cdot \*x) \right\} - \min_{\*p \in \probs_s} \left\{ \sum_{a\in\actions} \pi_{s,a} \cdot \*p_a\tr (\*r_{s,a} + \gamma\cdot \*y) \right\} \\
        &\ge \min_{\*p \in \probs_s}  \left\{ \sum_{a\in\actions} \left( \pi_{s,a} \cdot \*p_a\tr (\*r_{s,a} + \gamma\cdot \*x) \right) - \sum_{a\in\actions} \left( \pi_{s,a} \cdot \*p_a\tr (\*r_{s,a} + \gamma\cdot \*y) \right) \right\} \\
        &= \min_{\*p \in \probs_s} \left\{ \sum_{a\in\actions} \pi_{s,a} \cdot \gamma\cdot \*p_a\tr (\*x - \*y) \right\}  ~.
    \end{align*}
    The result follows by constructing the stochastic matrix $\*P$ such that its $s$-th row is $\sum_{a\in\actions} \pi_{s,a} \cdot \*p_a\tr$ where $\*p_a$ is the optimizer in the last minimization above.        
\end{proof}

\section{Bisection Algorithm with Quasi-Linear Time Complexity} \label{sec:bisection_linear}

We adapt \cref{alg:bisection} to determine the optimal solution to problem~\eqref{eq:simplified_opt} in quasi-linear time without dependence on any precision $\epsilon$. Recall that \cref{alg:homotopy} computes the breakpoints $(\xi^a_t)_t$, $t = 0, \ldots, T_a+1$ and objective values $(q^a_t)_t$, $t = 0, \ldots, T_a + 1$, $T_a \leq S^2$, of each function $q_a$, $a \in \actions$. Moreover, each inverse function $q^{-1}_a$ is also piecewise affine with breakpoints $(q^a_t)_t$, $t = 0, \ldots, T_a+1$ and corresponding function values $\xi_t^a = q^{-1}_a (q_t^a)$, as well as $q^{-1}_a (u) = \infty$ for $u < q_{T_a+1}^a$. We use this data as input for our revised bisection scheme in \cref{alg:bisection_linear}.

\begin{algorithm}[htb]
    \KwIn{Breakpoints $(q_t^a)_{t = 0, \ldots, T_a+1}$, of all functions $q_a$, $a \in \actions$}
    \KwOut{The optimal solution $u^\star$ to the problem~\eqref{eq:simplified_opt}}
    Combine $q_t^a$, $t = 0, \ldots, T_a$ and $a \in \actions$, to a single list $\mathcal{K} = (\hat{q}_1, \ldots, \hat{q}_K)$ in ascending order, omitting any duplicates \;
    ~ \\
    \tcp{Bisection search to find the optimal line segment $(k_{\min}, k_{\max})$}
    $k_{\min} \gets 1$; $k_{\max} \gets K$  \;
    \While{$k_{\max} - k_{\min} > 1$}{    
        Split $\{ k_{\min}, \ldots, k_{\max} \}$ in half: $k \gets \operatorname{round}((k_{\min} + k_{\max})/ 2)$ \;
        Calculate the budget required to achieve $u = \hat{q}_k$: $s \gets \sum_{a\in\actions} q_a^{-1}(\hat{q}_k)$ \;
        \eIf{$s \le \kappa$}{
            $u = \hat{q}_k$ is \emph{feasible}: update the feasible upper bound: $k_{\max} \gets k$ \;
        }{
            $u = \hat{q}_k$ is \emph{infeasible}: update the infeasible lower bound: $k_{\min} \gets k$ \;
        }
    }
    ~ \\
    \tcp{All $q^{-1}_a$ are affine on $(\hat{q}_{k_{\min}}, \hat{q}_{k_{\max}})$}
    $u_{\min} \gets \hat{q}_{k_{\min}}$; $\mspace{73mu}$ $u_{\max} \gets \hat{q}_{k_{\max}}$  \;
    $s_{\min} \gets \sum_{a\in\actions} q_a^{-1}(u_{\min})$;
    $s_{\max} \gets \sum_{a\in\actions} q_a^{-1}(u_{\max})$ \;
    $\alpha \gets (\kappa - s_{\min}) / (s_{\max} - s_{\min})$ \;
    $u\opt \leftarrow (1-\alpha) \cdot u_{\min} + \alpha \cdot u_{\max}$\;
    \Return{$u\opt$}
    \caption{Quasi-linear time bisection scheme for solving~\eqref{eq:bellman_s_rectangular}} \label{alg:bisection_linear}
\end{algorithm}

\cref{alg:bisection_linear} first combines all breakpoints $q^a_t$, $t = 0, \ldots T_a+1$ and $a \in \actions$, of the inverse functions $q^{-1}_a$, $a \in \actions$, to a single list $\mathcal{K}$ in ascending order. It then bisects on the indices of these breakpoints. The result is a breakpoint pair $(k_{\min}, k_{\max})$ satisfying $k_{\max} = k_{\min} + 1$ as well as $\kappa \in \left[ \sum_{a\in\actions} q_a^{-1}(\hat{q}_{k_{\min}}), \, \sum_{a\in\actions} q_a^{-1}(\hat{q}_{k_{\max}}) \right]$. Since none of the functions $q^{-1}_a$ have a breakpoint between $\hat{q}_{k_{\min}}$ and $\hat{q}_{k_{\max}}$, finding the optimal solution $u^\star$ to problem~\eqref{eq:bellman_s_rectangular} then reduces to solving a single linear equation in one unknown, which is done in the last part of \cref{alg:bisection_linear}.


The complexity of \cref{alg:bisection_linear} is dominated by the merging of the sorted lists $(q_t^a)_{t = 0, \ldots T_a+1}$, $a \in \mathcal{A}$, as well as the computation of $s$ inside the while-loop. Merging $A$ sorted lists, each of size less than or equal to $CS$, can be achieved in time $\bigO(CSA \log A)$. However, each one of these lists needs to be also sorted in \cref{alg:homotopy} giving the overall complexity of $\bigO (CSA \log CSA)$. Then, computing $q_a^{-1}$ at a given point can be achieved in time $\bigO (\log CS)$, so that $s$ in an individual iteration of the while-loop can be computed in time $\bigO (A \log CS)$. Since the while-loop is executed $\bigO (\log CSA)$ many times, computing $s$ has an overall complexity of $\bigO (A \log CS \log CSA)$. We thus conclude that \cref{alg:bisection_linear} has a complexity of $\bigO (CSA \log A + A \log CS \log CSA)$.

\section{Computing the Bellman Operator via Linear Programming} \label{sec:s_rect_linear_formulation}

In this section we present an LP formulation for the robust s-rectangular Bellman optimality operator $\Bell$ defined in~\eqref{eq:bellman_s_rectangular}:
\[
    (\Bell \b{v})_s = \max_{\bm{d} \in\Delta^A}  \min_{\*p\in (\Delta^\statecount)^A} \left\{ \sum_{a\in\actions} d_a \cdot \*p_a \tr \*z_a  \ss \sum_{a\in\actions} \lVert \b{p}_a - \bar{\b{p}}_{s,a} \rVert_{1,\b{w}_{s,a}} \le \kappa_{s} \right\}
\]
Here, we use $\*z_a = \*r_{s,a} + \gamma\cdot\*v$ in the objective function. Employing an epigraph reformulation, the inner minimization problem can be re-expressed as the following linear program:
\[
    \begin{array}{l@{\quad}l@{\qquad}l@{\qquad}l}
        \displaystyle \min_{\b{p}\in\Real^{A \times S},\b{\theta}\in\Real^{A \times S}} & \displaystyle \sum_{a\in\actions} d_a \cdot \b{z}_a\tr \b{p}_a \\[6mm]
        \displaystyle \text{subject to} & \displaystyle \one\tr \b{p}_a = 1 & \displaystyle \forall a \in \actions & [x_a] \\
        & \displaystyle \b{p}_a - \bar{\b{p}}_a \ge - \b{\theta}_a & \displaystyle \forall a \in \actions & [y^n_a] \\
        & \displaystyle \bar{\b{p}}_a - \b{p}_a \ge - \b{\theta}_a & \displaystyle \forall a \in \actions & [y^p_a] \\[1mm]
        & \displaystyle -\sum_{a\in\actions} \b{w}_a\tr \b{\theta}_a \ge - \kappa & & [\lambda] \\[4mm]
        & \displaystyle \b{p} \ge \zero, \quad \b{\theta} \ge \zero
    \end{array}
\]
For ease of exposition, we have added the dual variables corresponding to each constraint in brackets. This linear program is feasible by construction, which implies that its optimal value coincides with the optimal value of its dual. We can thus dualize this linear program and combine it with the outer maximization to obtain the following linear programming reformulation of the the robust s-rectangular Bellman optimality operator $\Bell$:
\[
    \begin{array}{l@{\quad}l@{\qquad}l@{\qquad}l}
        \displaystyle \max_{\substack{\b{d}\in\Real^{A},\b{x}\in\Real^{A}, \, \lambda\in\Real \\ \b{y^p}\in\Real^{S\times A},\b{y^n}\in\Real^{S\times A}}} & \displaystyle \sum_{a\in\actions} \Bigl( x_a + \bar{\b{p}}_a\tr [\b{y}^n_a - \b{y}^p_a] \Bigr) - \kappa \cdot \lambda \\[6mm]
        \displaystyle \text{subject to} & \displaystyle \one\tr \b{d} = 1, \quad \b{d} \ge \zero \\
        & \displaystyle - \b{y}^p_a + \b{y}^n_a + x \cdot \one \le d_a z_a & \displaystyle \forall a\in\actions \\
        & \displaystyle \b{y}^p_a + \b{y}^n_a - \lambda \cdot \b{w}_a  \le \zero & \displaystyle \forall a\in\actions \\
        & \displaystyle \b{y}^p \ge \zero \quad \b{y}^n \ge \zero \\
        & \displaystyle \lambda \ge 0
    \end{array}
\]
This problem has $\bigO (SA)$ variables and an input bitlength of $\bigO (SA)$. As such, its theoretical runtime complexity is $\bigO (S^{4.5} A^{4.5})$.

\end{document}